\documentclass[twoside,11pt]{article}

\usepackage[accepted]{melba}

\usepackage{mwe} 

\usepackage{amsmath,amsfonts}
\usepackage{booktabs}
\usepackage{bm}

\newcommand{\expnumber}[2]{{#1}\mathrm{e}{#2}}

\melbaid{2024:026}  
\doi{https://doi.org/10.59275/j.melba.2024-276b}
\melbaauthors{Honkamaa and Marttinen}  
\volume{2}
\firstpageno{2148}  
\melbayear{2024}  
\datesubmitted{01/2024}  
\datepublished{11/2024}  

\ShortHeadings{Symmetric, inverse consistent, and topology preserving image registration}{Honkamaa and Marttinen}

\title{SITReg: Multi-resolution architecture for symmetric, inverse consistent, and topology preserving image registration}

\author{\firstname Joel \surname Honkamaa\orcid{0000-0003-1532-9848} \email joel.honkamaa@aalto.fi \\  
	\addr Department of Computer Science, Aalto University, Finland
	\AND
	\name Pekka Marttinen\orcid{0000-0001-7078-7927} \email pekka.marttinen@aalto.fi \\
	\addr Department of Computer Science, Aalto University, Finland
}

\begin{document}

\maketitle

\begin{abstract}
	Deep learning has emerged as a strong alternative for classical iterative methods for deformable medical image registration, where the goal is to find a mapping between the coordinate systems of two images. Popular classical image registration methods enforce the useful inductive biases of symmetricity, inverse consistency, and topology preservation by construction. However, while many deep learning registration methods encourage these properties via loss functions, no earlier methods enforce all of them by construction. Here, we propose a novel registration architecture based on extracting multi-resolution feature representations which is by construction symmetric, inverse consistent, and topology preserving. We also develop an implicit layer for memory efficient inversion of the deformation fields. Our method achieves state-of-the-art registration accuracy on three datasets.
	The code is available at~\url{https://github.com/honkamj/SITReg}.
\end{abstract}

\section{Introduction}\label{sec:intro}

Deformable medical image registration aims at finding a mapping between coordinate systems of two images, called a \textit{deformation}, to align them anatomically. Deep learning can be used to train a registration network which takes as input two images and outputs a deformation. We focus on unsupervised intra-modality registration without a ground-truth deformation and where images are of the same modality, applicable, e.g., when deforming brain MRI images from different patients to an atlas or analyzing a patient's breathing cycle from multiple images.

Success of the most popular deep learning architectures is often seen through the desirable inductive biases imposed by suitable geometric priors (such as translation equivariance for CNNs) \citep{bronstein2017geometric, bronstein2021geometric}. In image registration priors of \textit{inverse consistency}, \textit{symmetry}, and \textit{topology preservation} are widely considered to be useful \citep{sotiras2013deformable}. While some of the most popular classical methods enforce all of these properties by construction \citep{ashburner2007fast, avants2008symmetric}, no prior deep learning methods do, and we address this gap (see a detailed literature review in Appendix \ref{appendix:related_work}). To clearly state our contributions, we start by defining the properties (further clarifications in Appendix \ref{appendix:property_examples}).

We define a \textit{registration method} as a function $f$ that takes two images, say $x_A$ and $x_B$, and produces a deformation. In general one can predict the deformation with any method in either direction by varying the input order, but some methods predict both forward and inverse deformations directly, and we use subscripts to indicate the direction for such methods. For example, $f_{1\to2}$ produces a deformation that aligns the image of the first argument to the image of the second argument. Note that even if $f_{1\to 2}(x_A, x_B)$ and $f_{2\to 1}(x_B, x_A)$ both denote a deformation for aligning image $x_A$ to $x_B$, they in general can be different functions. As a result, a registration method may predict up to four different deformations for any given input pair: $f_{1\to2}(x_A, x_B)$, $f_{2\to1}(x_A, x_B)$, $f_{1\to2}(x_B, x_A)$, and $f_{2\to1}(x_B, x_A)$. If a method predicts a deformation only in one direction for a given input order, we might omit the subscript.

\textit{Inverse consistent} registration methods ensure that $f_{1\to2}(x_A, x_B)$ is an accurate inverse of $f_{2\to1}(x_A, x_B)$, which we quantify using the \textit{inverse consistency error}: $||f_{1\to2}(x_A, x_B) \circ f_{2\to1}(x_A, x_B) - \mathcal{I}||^2$, where $\circ$ is the composition operator and $\mathcal{I}$ is the identity deformation. Inverse consistency has been enforced using a loss or by construction. However, due to a limited spatial resolution of the predicted deformations, even for the by construction inverse consistent methods the inverse consistency error is not exactly zero \citep{ashburner2007fast}.

\begin{figure}
\centering
\includegraphics[width=1.0\textwidth]{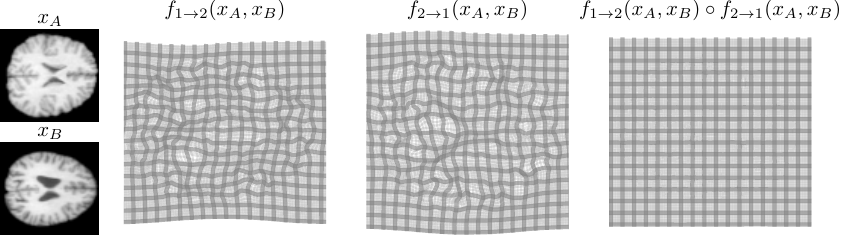}
\caption{\textbf{Example deformation from the method.} \textit{Left:} Forward deformation. \textit{Middle:} Inverse deformation. \textit{Right:} Composition of the forward and inverse deformations. Only one 2D slice is shown of the 3D deformation. The deformation is from the LPBA40 experiment. For more detailed visualization of a predicted deformation, see Figure \ref{appendix-fig:detailed_deformation_example} in Appendix \ref{appendix:result_visualization}.}
\label{fig:example_deformation}
\end{figure}

In \textit{symmetric registration}, the registration outcome does not depend on the order of the inputs, i.e., $f_{1\to2}(x_A, x_B)$ equals $f_{2\to1}(x_B, x_A)$. Unlike with inverse consistency, $f_{1\to2}(x_A, x_B)$ can equal $f_{2\to1}(x_B, x_A)$ exactly \citep{avants2008symmetric, estienne2021mics}, which we call \textit{symmetric by construction}. A related property, cycle consistency, can be assessed using \textit{cycle consistency error} $||f(x_A, x_B) \circ f(x_B, x_A) - \mathcal{I}||^2$. It can be computed for any method since it does not require the method to predict deformations in both directions. If the method is symmetric by construction, inverse consistency error equals cycle consistency error.

We define \textit{topology preservation} of predicted deformations similarly to \citet{christensen11995topological}. From the real-world point of view this means the preservation of anatomical structures, preventing non-smooth changes. Mathematically we want the deformations to be homeomorphisms, i.e., invertible and continuous. In registration literature it is common to talk about diffeomorphims which are additionally differentiable. In practice we want a deformation not to fold on top of itself which we measure by estimating the local Jacobian determinants of the predicted deformations, and checking whether they are positive.

With these definitions at hand, we summarize our main contributions as follows:
\begin{itemize}
    \item We propose a novel multi-resolution deep learning registration architecture which is by construction inverse consistent, symmetric, and preserves topology. The properties are fulfilled for the whole multi-resolution pipeline, not just separately for each resolution. Apart from the parallel works by \citep{greer2023inverse, zhang2023symmetric}, we are not aware of other multi-resolution deep learning registration methods which are by construction both symmetric and inverse consistent. For motivation of the multi-resolution approach, see Section \ref{sec:multi-resolution_background}.
    \item As a component in our architecture, we propose an \textit{implicit} neural network layer, which we call \textit{deformation inversion layer}, based on a well-known fixed point iteration formula \citep{chen2008simple} and recent advances in Deep Equilibrium models \citep{bai2019deep, duvenaud2020deep}. The layer allows memory efficient inversion of deformation fields.
    \item We show that the method achieves state-of-the-art results on two brain subject-to-subject registration datasets, and state-of-the-art results for deep learning methods on inspiration-exhale registration of lung CT scans.
\end{itemize}

We name the method \textit{SITReg} after its symmetricity, inverse consistency and topology preservation properties.

\section{Background and preliminaries}

\subsection{Topology preserving registration}\label{sec:topology_preserving_background}

The diffeomorphic LDDMM framework \citep{cao2005large} was the first approach suggested for by construction topology preserving registration. The core idea was to generate diffeomorphisms through integration of time-varying velocity fields constrained by certain differential equations. \citet{arsigny2006log} suggested more constrained but computationally cheaper stationary velocity field (SVF) formulation and it was later adopted by popular registration algorithms \citep{ashburner2007fast, vercauteren2009diffeomorphic}. While some unsupervised deep learning methods do use the LDDMM approach for generating topology preserving deformations \citep{shen2019region, ramon2022lddmm, wang2023metamorph}, most commonly \citep{chen2023survey} topology preservation in unsupervised deep learning is achieved using the more efficient SVF formulation, e.g. by \citet{krebs2018unsupervised, krebs2019learning, niethammer2019metric, shen2019networks, shen2019region, mok2020fast}.

Another classical method by \citet{choi2000injectivity, rueckert2006diffeomorphic} generates topology preserving deformations by constraining each deformation to be diffeomorphic but small, and forming the final deformation as a composition of multiple small deformations. Since diffeomorphisms form a group under composition, the final deformation is diffeomorphic. The principle is close to a practical implementation of the SVF, where the velocity field is integrated by first scaling it down by a power of two and interpreting the result as a small deformation, which is then repeatedly composed with itself. The idea is hence similar: a composition of small deformations.

In this work we build topology preserving deformations using the same strategy, as a composition of small topology preserving deformations.

\subsection{Inverse consistent registration}

Originally inverse consistency was achieved via variational losses \citep{christensen11995topological} but later LDDMM and SVF frameworks allowed for inverse consistent by construction methods since the inverse can be obtained by integrating the velocity field in the opposite direction. Both approaches are found among the deep learning methods as well, as some enforce inverse consistency via a penalty \citep{zhang2018inverse, kim2019unsupervised, estienne2021mics}, and many use the SVF formulation as mentioned in Section \ref{sec:topology_preserving_background}.

Compared to the earlier deep learning approaches, we take a methodologically slightly different approach of using the proposed deformation inversion layer for building an inverse consistent architecture.

\subsection{Symmetric registration}\label{sec:symmetric_registration_background}

Symmetric registration methods consider both images equally: swapping the input order should not change the registration outcome. Developing symmetric registration methods has a long history \citep{sotiras2013deformable}, but one particularly relevant for this work is a method called symmetric normalization \citep[SyN,][]{avants2008symmetric} which learns two separate transformations: one for deforming the first image half-way toward the second image and the other for deforming the second image half-way toward the first image. The images are matched in the intermediate coordinates and the full deformation is obtained as a composition of the half-way deformations (one of which is inverted). The same idea was applied in deep learning setting by SYMNet \citep{mok2020fast}. However, SYMNet does not guarantee symmetricity by construction during inference (see Figure \ref{appendix-fig:cycle_consistency_visualization} in Appendix \ref{appendix:result_visualization}). Some existing deep learning registration methods enforce cycle consistency via a penalty \citep{mahapatra2019training, gu2020pair, zheng2021symreg}, and the method by \citet{estienne2021mics} is symmetric by construction but only for a single component of their multi-step formulation, and not inverse consistent by construction. Recently, parallel with and unrelated to us, \citet{iglesias2023ready, greer2023inverse, zhang2023symmetric} have proposed by construction symmetric and inverse consistent registration methods within the SVF framework, in a different way from us.

We use the idea of deforming the images half-way towards each other to achieve symmetry throughout our multi-resolution architecture.

\subsection{Multi-resolution registration}\label{sec:multi-resolution_background}

Multi-resolution registration methods learn the deformation by first estimating it in a low resolution and then incrementally improving it while increasing the resolution. For each resolution one feeds the input images deformed with the deformation learned thus far, and incrementally composes the full deformation. Since its introduction a few decades ago \citep{rueckert1999nonrigid, oliveira2014medical}, the approach has been used in the top-performing classical and deep learning registration methods \citep{avants2008symmetric, klein2009evaluation, mok2020large, mok2021conditional, hering2022learn2reg}.

In this work we propose the first multi-resolution deep learning registration architecture that is by construction symmetric, inverse consistent, and topology preserving.

\subsection{Deep equilibrium networks}\label{sec:deqn}

Deep equilibrium networks use \textit{implicit} fixed point iteration layers, which have emerged as an alternative to the common \textit{explicit} layers \citep{bai2019deep, bai2020multiscale, duvenaud2020deep}. Unlike explicit layers, which produce output via an exact sequence of operations, the output of an implicit layer is defined indirectly as a solution to a fixed point equation, which is specified using a fixed point mapping. In the simplest case the fixed point mapping takes two arguments, one of which is the input. For example, let  $g: A \times B \to B$ be a fixed point mapping defining an implicit layer. Then, for a given input $a$, the output of the layer is the solution $z$ to equation
\begin{equation}
    z = g(z, a).
\end{equation}
This equation is called a fixed point equation and the solution is called a fixed point solution. If $g$ has suitable properties, the equation can be solved iteratively by starting with an initial guess and repeatedly feeding the output as the next input to $g$. More advanced iteration methods have also been developed for solving fixed point equations, such as Anderson acceleration \citep{walker2011anderson}.

The main mathematical innovation related to deep equilibrium networks is that the derivative of an implicit layer w.r.t. its inputs can be calculated based solely on a fixed point solution, i.e., no intermediate iteration values need to be stored for back-propagation. Now, given some solution $(a_0, z_0)$, such that $z_0 = g(z_0, a_0)$, and assuming certain local invertibility properties for $g$, the implicit function theorem says that there exists a solution mapping in the neighborhood of $(a_0, z_0)$, which maps other inputs to their corresponding solutions. Let us denote the solution mapping as $z^*$. The solution mapping can be seen as the theoretical explicit layer corresponding to the implicit layer. To find the derivatives of the implicit layer we need to find the Jacobian of $z^*$ at point $a_0$ which can be obtained using implicit differentiation as
\begin{equation*}
\partial z^*(a_0) = \left[ I - \partial_1 g(z_0, a_0) \right]^{-1}\partial_0 g(z_0, a_0).
\end{equation*}
The vector-Jacobian product of $z^*$ needed for back-propagation can be calculated using another fixed point equation without fully computing the Jacobians, see, e.g., \citet{duvenaud2020deep}. Hence, both forward and backward passes of the implicit layer can be computed as a fixed point iteration.

We use these ideas to develop a neural network layer for inverting deformations based on the fixed point equation, following \citet{chen2008simple}. The layer is very memory efficient as only the fixed point solution needs to be stored for the backward pass.

\section{Methods}

Let $n$ denote the dimensionality of the image, e.g., $n = 3$ for 3D medical images, and $k$ the number of channels, e.g., $k = 3$ for an RGB-image. The goal in deformable image registration is to find a mapping from $\mathbb{R}^n$ to $\mathbb{R}^n$, connecting the coordinate systems of two non-aligned images  $x_A, x_B: \mathbb{R}^n\to \mathbb{R}^k$, called a deformation. Application of a deformation to an image can be mathematically represented as a (function) composition of the image and the deformation, denoted by $\circ$. Furthermore, in practice linear interpolation is used to represent images (and deformations) in continuous coordinates.

In this work the deformations are in practice stored as displacement fields with the same resolution as the registered images, that is, each pixel or voxel is associated with a displacement vector describing the coordinate difference between the original image and the deformed image (e.g. if $n=3$, displacement field is tensor with shape $3 \times H \times W \times D$ where $H \times W \times D$ is the shape of the image). In our notation we equate the displacement fields with the corresponding coordinate mappings, and always use $\circ$ to denote the deformation operation (sometimes called warping).

In deep learning based image registration, we aim at learning a \textit{neural network} $f$ that takes two images as input and outputs a mapping between the image coordinates. Specifically, in medical context $f$ should be such that $x_A \circ f(x_A, x_B)$ matches anatomically with $x_B$.

\subsection{Symmetric formulation}\label{sec:symmetric_formulation_introduction}

As discussed in Section \ref{sec:intro}, we want our method to be cycle consistent. That is, since in the ideal case of $f$ finding the correct coordinate mapping between any given inputs $x_A$ and $x_B$, $f(x_A, x_B)$ and $f(x_B, x_A)$ should be inverses of each other. Hence enforcing such a property by construction should provide a useful constraint on the optimization space. To achieve this, we define the network $f$ using another \textit{auxiliary network} $u$ which also predicts deformations:
\begin{equation}\label{eq:anti-symmetric_formulation}
    f(x_A, x_B) := u(x_A, x_B) \circ u(x_B, x_A)^{-1}.
\end{equation}
By defining $f$ this way, it holds by construction that $f(x_A, x_B) = f(x_B, x_A)^{-1}$ apart from small numerical errors introduced by the composition and inversion. An additional benefit is that $f(x_A, x_A)$ equals the identity transformation, again apart from numerical inaccuracies, which is a natural requirement for a registration method.
Applying the formulation in Equation \ref{eq:anti-symmetric_formulation} naively would double the computational cost. To avoid this we encode features from the inputs separately before feeding them to the deformation extraction network in Equation \ref{eq:anti-symmetric_formulation}. A similar approach has been used in recent registration methods \citep{estienne2021mics, young2022superwarp}. Denoting the feature extraction network by $h$, the modified formulation is
\begin{equation}\label{eq:anti-symmetric_formulation_with_features}
    f(x_A, x_B) := u(h(x_A), h(x_B)) \circ u(h(x_B), h(x_A))^{-1}.
\end{equation}

\subsection{Multi-resolution architecture}\label{sec:multi-resolution}

\begin{figure*}[t]
\centering
\includegraphics[width=1.0\textwidth]{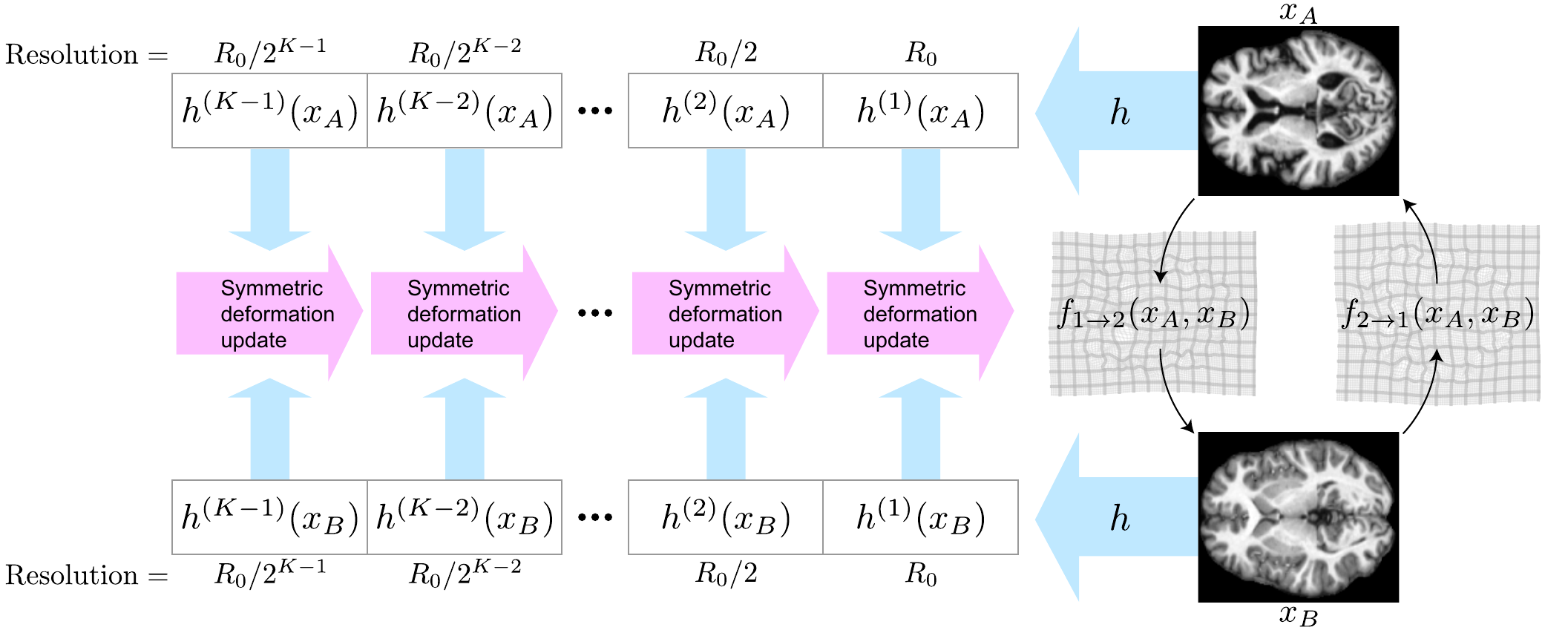}
\caption{\textbf{Overview of the proposed architecture.} Multi-resolution features are first extracted from the inputs $x_A$ and $x_B$ using convolutional encoder $h$. Output deformations $f_{1\to2}(x_A, x_B)$ and $f_{2\to1}(x_A, x_B)$ are built recursively from the multi-resolution features using the symmetric deformation updates described in Section \ref{sec:multi-resolution} and visualized in Figure \ref{fig:anti-symmetric_update}. The architecture is symmetric and inverse consistent with respect to the inputs and the final deformation is obtained in both directions. The brain images are from the OASIS dataset \citep{marcus2007open}}
\label{fig:architecture_overview}
\end{figure*}

As the overarching architecture, we propose a novel symmetric and inverse consistent multi-resolution coarse-to-fine approach. For motivation, see Section \ref{sec:multi-resolution_background}. Overview of the architecture is shown in Figure \ref{fig:architecture_overview}, and the prediction process is demonstrated visually in Figure \ref{appendix-fig:detailed_deformation_example} (Appendix \ref{appendix:result_visualization}).

First, we extract image feature representations $h^{(k)}(x_A), h^{(k)}(x_B)$, at different resolutions $k \in \{0,\dots,K - 1\}$. Index $k=0$ is the original resolution and increasing $k$ by one halves the spatial resolution. In practice $h$ is a ResNet \citep{he2016deep} style convolutional network and features at each resolution are extracted sequentially from previous features. Starting from the lowest resolution $k=K-1$, we recursively build the final deformation between the inputs using the extracted representations. To ensure symmetry, we build two deformations: one deforming the first image half-way towards the second image, and the other for deforming the second image half-way towards the first image (see Section \ref{sec:symmetric_registration_background}). The full deformation is composed of these at the final stage. Let us denote the half-way deformations extracted at resolution $k$ as $d_{1\to1.5}^{(k)}$ and $d_{2\to1.5}^{(k)}$. Initially, at level $k = K$, these are identity deformations. Then, at each $k = K-1,\ldots, 0$, the half-way deformations are updated by composing them with a predicted update deformation. In detail, the update at level $k$ consists of three steps (visualized in Figure \ref{fig:anti-symmetric_update}):
    \begin{enumerate}
        \item Deform the feature representations $h^{(k)}(x_A), h^{(k)}(x_B)$ of level $k$ towards each other by the half-way deformations from the previous level $k + 1$:
        \begin{equation}\label{eq:deformed_features}
            z_1^{(k)}:=h^{(k)}(x_A) \circ d_{1\to1.5}^{(k + 1)}\ \quad \text{and} \quad \ z_2^{(k)}:=h^{(k)}(x_B) \circ d_{2\to1.5}^{(k + 1)}.
        \end{equation}
        Note that the deformations have a higher (or the same when $k=0$) spatial resolution than the deformed feature volumes, and in practice we first downsample (with linear interpolation) the deformation to the resolution of the feature volume, and then resample the feature volume based on the downsampled deformation.
        \item Define an \textit{update deformation} $\delta^{(k)}$, using the idea from Equation \ref{eq:anti-symmetric_formulation_with_features} and the half-way deformed feature representations $z_1^{(k)}$ and $z_2^{(k)}$:
        \begin{equation}
            \delta^{(k)} := u^{(k)}(z_1^{(k)}, z_2^{(k)}) \circ u^{(k)}(z_2^{(k)}, z_1^{(k)})^{-1}.\label{eq:update_deformation_formula}
        \end{equation}
        Here, $u^{(k)}$ is a trainable convolutional neural network predicting an invertible auxiliary deformation (details in Section \ref{sec:topology_preserving_u}). The intuition here is that the symmetrically predicted update deformation $\delta^{(k)}$ should learn to adjust for whatever differences in the image features remain after deforming them half-way towards each other in Step 1 with deformations $d^{(k+1)}$ from the previous resolution.
        \item Obtain the updated half-way deformation $d_{1\to1.5}^{(k)}$ by composing the earlier half-way deformation of level $k + 1$ with the update deformation $\delta^{(k)}$
        \begin{equation}\label{eq:forward_update}
            d_{1\to1.5}^{(k)} := d_{1\to1.5}^{(k + 1)}\ \circ\ \delta^{(k)}.
        \end{equation}
        For the other direction $d_{2\to1.5}^{(k)}$, we use the inverse of the deformation update $\left(\delta^{(k)}\right)^{-1}$ which can be obtained simply by reversing $z_1^{(k)}$ and $z_2^{(k)}$ in Equation \ref{eq:update_deformation_formula} (see Figure \ref{fig:anti-symmetric_update}):
        \begin{equation}\label{eq:inverse_update}
            d_{2\to1.5}^{(k)}\\
                = d_{2\to1.5}^{(k + 1)}\ \circ\ \left(\delta^{(k)}\right)^{-1}.
        \end{equation}
        The inverses $\left(d_{1\to1.5}^{(k)}\right)^{-1}$ and $\left(d_{2\to1.5}^{(k)}\right)^{-1}$ are updated similarly.
    \end{enumerate}

The full registration architecture is then defined by the functions $f_{1\to2}$ and $f_{2\to1}$ which compose the half-way deformations from stage $k=0$:
    \begin{equation}\label{eq:final_deformation}
        f_{1\to2}(x_A, x_B) := d_{1\to1.5}^{(0)} \circ \left(d_{2\to1.5}^{(0)}\right)^{-1}
        \ \quad \text{and}\ \quad
        f_{2\to1}(x_A, x_B) := d_{2\to1.5}^{(0)} \circ \left(d_{1\to1.5}^{(0)}\right)^{-1}.
    \end{equation}
Note that $d_{1\to1.5}^{(0)}$, $d_{2\to1.5}^{(0)}$, and their inverses 
are functions of $x_A$ and $x_B$ through the features $h^{(k)}(x_A), h^{(k)}(x_B)$ in Equation \ref{eq:deformed_features}, but the dependence is suppressed in the notation for clarity.

\begin{figure}
\centering
\includegraphics[width=0.7\columnwidth]{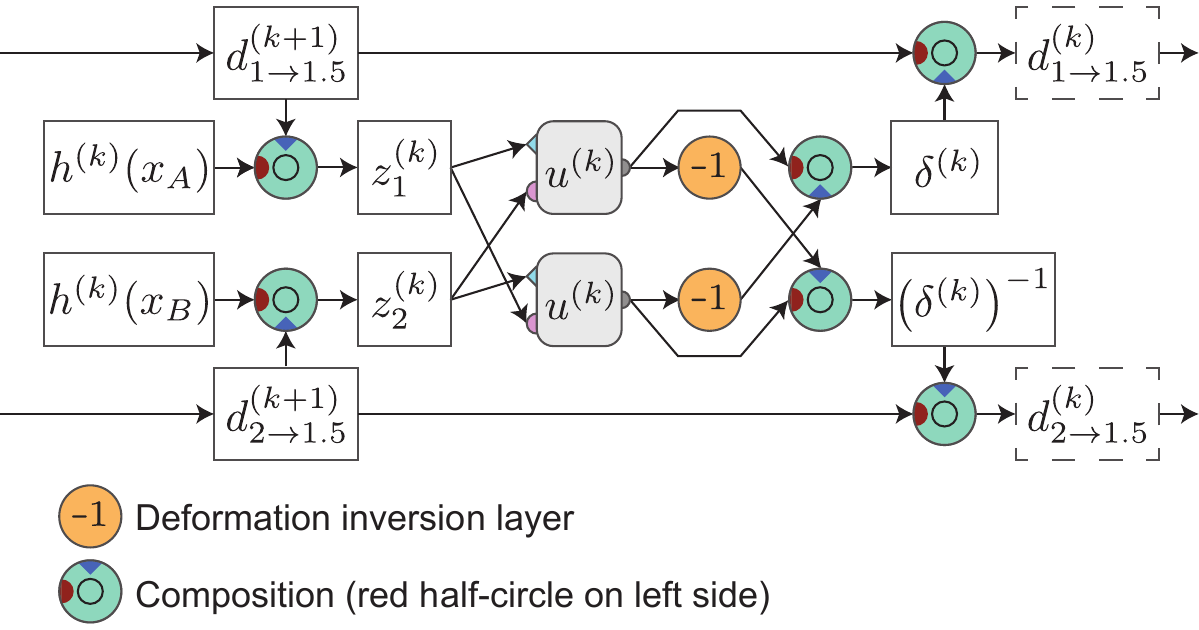}
\caption{\textbf{Recursive multi-resolution deformation update.} The deformation update at resolution $k$, described in Section \ref{sec:multi-resolution}, takes as input the half-way deformations $d_{1\to1.5}^{(k+1)}$ and $d_{2\to1.5}^{(k+1)}$ from the previous resolution, and updates them through a composition with an update deformation $\delta^{(k)}$. The update deformation  $\delta^{(k)}$ is calculated symmetrically from image features $z_1^{(k)}$ and $z_2^{(k)}$ (deformed mid-way towards each other with the previous half-way deformations) using a neural network $u^{(k)}$ according to Equation \ref{eq:update_deformation_formula}. The deformation inversion layer for inverting auxiliary deformations predicted by $u^{(k)}$ is described in Section \ref{sec:def_inv_layer}.}
\label{fig:anti-symmetric_update}
\end{figure}

By using half-way deformations at each stage, we avoid the problem with full deformations of having to select either of the image coordinates to which to deform the feature representations of the next stage, breaking the symmetry of the architecture. Now we can instead deform the feature representations of both inputs by the symmetrically predicted half-way deformations, which ensures that the updated deformations after each stage are separately invariant to input order.

\subsection{Implicit deformation inversion layer}\label{sec:def_inv_layer}

Implementing the architecture requires inverting deformations from $u^{(k)}$ in Equation \ref{eq:update_deformation_formula}. This could be done, e.g., with the SVF framework, but we propose an approach which requires storing $\approx 5$ times less data for the backward pass than the standard SVF. The memory saving is significant due to the high memory consumption of volumetric data, allowing larger images to be registered. During each forward pass $2 \times (K - 1)$ inversions are required. More details are provided in Appendix \ref{appendix:memory}.

As shown by \citet{chen2008simple}, deformations can be inverted in certain cases by a  fixed point iteration. Consequently, we propose to use the deep equilibrium network framework from Section \ref{sec:deqn} for inverting deformations, and label the resulting layer \textit{deformation inversion layer}. The fixed point equation proposed by \citet{chen2008simple} is
\begin{equation*}
    g(z, a) := -(a - \mathcal{I}) \circ z + \mathcal{I},
\end{equation*}
where $a$ is the deformation to be inverted, $z$ is the candidate for the inverse of $a$, and $\mathcal{I}$ is the identity deformation. It is easy to see that substituting $a^{-1}$ for $z$, yields $a^{-1}$ as output. We use Anderson acceleration \citep{walker2011anderson} for solving the fixed point equation and use the memory-effecient back-propagation \citep{bai2019deep, duvenaud2020deep} strategy discussed in Section \ref{sec:deqn}.

Lipschitz condition is sufficient for the fixed point algorithm to converge \citep{chen2008simple}, and we ensure that the predicted deformations fulfill the condition (see Section \ref{sec:topology_preserving_u}). The iteration converges well also in practice as shown in Appendix \ref{appendix:deformation_inversion_iteration_counts}.

\subsection{Topology preserving deformation prediction networks}\label{sec:topology_preserving_u}

Each $u^{(k)}$ predicts a deformation based on the features $z_1^{(k)}$ and $z_2^{(k)}$ and we define the networks $u^{(k)}$ as CNNs predicting cubic spline control point grid in the resolution of the features $z_1^{(k)}$ and $z_2^{(k)}$. The cubic spine control point grid can then be used to generate the displacement field representing the deformation. The use of cubic spline control point grid for representing deformations is a well-known strategy in image registration, see e.g. \citep{rueckert2006diffeomorphic, de2019deep}.

The deformations generated by $u^{(k)}$ have to be invertible to ensure the topology preservation property, and particularly invertible by the deformation inversion layer. To ensure that, we limit the control point absolute values below a hard constraint $\gamma^{(k)}$ using a scaled $\operatorname{Tanh}$ function.

In more detail, each $u^{(k)}$ consists of the following sequential steps:
\begin{enumerate}
    \item Concatenation of the two inputs, $z_1^{(k)}$ and $z_2^{(k)}$, along the channel dimension. Before concatenation we reparametrize the features as $z_1^{(k)} - z_2^{(k)}$ and $z_1^{(k)} + z_2^{(k)}$ as suggested by \citet{young2022superwarp}.
    \item Any number of spatial resolution preserving ($\text{stride}=1$ and $\text{padding}=\text{same}$) convolutions with an activation after each of the convolutions except the final one. The number of output channels of the final convolutions should equal the dimensionality ($3$ in our experiments).
    \item\label{appendix-enumerate:tanh_step} $\gamma^{(k)} \times \operatorname{Tanh}$ function
    \item Cubic spline upsampling to the image resolution by interpreting the output of the step \ref{appendix-enumerate:tanh_step} as cubic spline control point grid, similarly to e.g. \cite{de2019deep}. Cubic spline upsampling can be effeciently implemented as one dimensional transposed convolutions along each axis. 
\end{enumerate}

As shown in Appendix \ref{appendix:optimal_gamma}, the optimal upper bound for $\gamma^{(k)}$ ensuring invertibility by the deformation inversion layer but also in general, can be obtained by the formula $\gamma^{(k)} < \frac{1}{K^{{k}}_n}$
where
\begin{equation}\label{eq:optimal_gamma}
K^{(k)}_n := \max_{\begin{subarray}{l}x \in X\end{subarray}}\ \sum_{\alpha\in \mathbb{Z}^n} \left| \sum_{j=1}^{n}\frac{B(x_j + \frac{1}{2^k} - \alpha_j) - B(x_j - \alpha_j)}{1 / 2^k} \prod_{i\in N\setminus \{j\}} B(x_i - \alpha_i)\right|,
\end{equation}
$n$ is the dimensionality of the images (in this work $n = 3$),  $X := \{\frac{1}{2}+ \frac{1}{2^{k + 1}} + \frac{i}{2^k}\ |\ i \in \mathbb{Z}\}^n \cap [0, 1]^n$ are the relative sampling positions used in cubic spline upsampling (imlementation detail), and $B$ is a centered cardinal cubic B-spline (symmetric function with finite support). In practice we define $\gamma^{(k)} := 0.99 \times \frac{1}{K^{(k)}_n}$.

The formula can be evaluated exactly for dimensions $n = 2, 3$ for any practical number of resolution levels (for concrete values, see Table \ref{appendix-table:K_values} in Appendix \ref{appendix:optimal_gamma}). Note that the outer sum over $\mathbb{Z}^n$ is finite since $B$ has a finite support and hence only a finite number of terms are non-zero.

The strategy of limiting the absolute values of cubic spline control points to ensure invertibility is similar to the one taken by \citet{rueckert2006diffeomorphic} based on the proof by \citet{choi2000injectivity}. However, our mathematical proof additionally covers the convergence by the deformation inversion layer, and provides an exact bound even in our discretised setting (see Appendix \ref{appendix:optimal_gamma} for more details).

\subsection{Theoretical properties}

\begin{theorem}\label{theorem:inverse_consistent}
The proposed architecture is inverse consistent by construction.
\end{theorem}

\begin{theorem}\label{theorem:symmetric}
The proposed architecture is symmetric by construction.
\end{theorem}

\begin{theorem}\label{theorem:topology_preserving}
The proposed architecture is topology preserving. 
\end{theorem}

\begin{proof}\renewcommand{\BlackBox}{}
Appendix \ref{appendix:theoretical_properties_proof}, including discussion on numerical errors caused by limited sampling resolution.
\end{proof}

\subsection{Training and implementation}

We train the model in an unsupervised end-to-end manner similarly to most other unsupervised registration methods, by using similarity and deformation regularization losses. The similarity loss encourages deformed images to be similar to the target images, and the regularity loss encourages desirable properties, such as smoothness, on the predicted deformations. For similarity we use local normalized cross-correlation with window width 7 and for regularization we use $L^2$ penalty on the gradients of the displacement fields, identically to VoxelMorph \citep{balakrishnan2019voxelmorph}. We apply the losses in both directions to maintain symmetry. One could apply the losses in the intermediate coordinates and avoid building the full deformations during training. 
The final loss is:
\begin{equation}
    \mathcal{L} = \operatorname{NCC}(x_A \circ d_{1\to 2},\ x_B) +
    \operatorname{NCC}(x_A,\ x_B\circ d_{2\to 1}) + \lambda \times \left[
        \operatorname{Grad}(d_{1\to 2}) + \operatorname{Grad}(d_{2\to 1})
    \right],
\end{equation}
where $d_{1\to2} := f_{1\to2}(x_A, x_B)$, $d_{2\to1} := f_{2\to1}(x_A, x_B)$,  $\operatorname{NCC}$ the local normalized cross-correlation loss, $\operatorname{Grad}$ the $L^2$ penalty on the gradients of the displacement fields, and $\lambda$ is the regularization weight. For details on hyperparameter selection, see Appendix \ref{appendix:hyperparameter_selection}.
Our implementation is in PyTorch \citep{pytorch}, and is available at ~\url{https://github.com/honkamj/SITReg}. Evaluation methods and preprocessing done by us, see Section \ref{sec:experiments}, are included.

\subsection{Inference}\label{sec:inference}

We consider two variants: \textbf{Standard}: The final deformation is formed by iteratively resampling at each image resolution (common approach). \textbf{Complete}: All individual deformations (outputs of $u^{(k)}$) are stored in memory and the final deformation is their true composition.
The latter is included only to demonstrate that the deformation is everywhere invertible (no negative determinants) without numerical sampling errors, but the first one is used unless stated otherwise, and perfect invertibility is not necessary in practice. Due to limited sampling resolution even the existing "diffeomorphic" registration frameworks such as SVF do not usually achieve perfect invertibility. 

\section{Experimental setup}\label{sec:experiments}

We evaluate our method on subject-to-subject registration of brain MRI images.

\subsection{Datasets}

We evaluate our method on two tasks: subject-to-subject registration of brain images and on inspiration-exhale registration of lung CT scans. The latter is considered challenging for deep learning methods, and classical optimization based methods remain state-of-the-art \citep{falta2024lung250m}.

We use two subject-to-subject registration datasets and evaluate on both of them separately: \textit{OASIS} brains dataset with 414 T1-weighted brain MRI images \citep{marcus2007open} as pre-processed for Learn2Reg challenge \mbox{\citep{hoopes2021hypermorph, hering2022learn2reg}} 
\footnote{\href{https://www.oasis-brains.org/\#access}{https://www.oasis-brains.org/\#access}}
, and \textit{LPBA40} dataset from University of California Laboratory of Neuro Imaging (USC LONI) with 40 brain MRI images \citep{shattuck2008construction}
\footnote{\href{https://resource.loni.usc.edu/resources/atlases/license-agreement/}{https://resource.loni.usc.edu/resources/atlases/license-agreement/}}

Pre-processing for both brain subject-to-subject datasets includes bias field correction, normalization, and cropping. For OASIS dataset we use affinely pre-aligned images and for LPBA40 dataset we use rigidly pre-aligned images. Additionally we train the models without any pre-alignment on OASIS data (\textit{OASIS raw}) to compare the methods with larger initial displacements. We crop the images in affinely pre-aligned OASIS dataset to $144 \times 192 \times 160$ resolution and in LPBA40 dataset to $160 \times 192 \times 160$ resolution. Images in raw OASIS dataset have resolution $256 \times 256 \times 256$ and we do not crop the images. Voxel sizes of the affinely aligned and raw datasets are the same. We split the OASIS dataset into $255$, $20$ and $139$ images for training, validation, and testing. The split differs from the Learn2Reg challenge since the test set is not available, but sizes correspond to the splits used by \citet{mok2020fast, mok2020large, mok2021conditional}. We used all image pairs for testing and validation, yielding $9591$ test and $190$ validation pairs. LPBA40 dataset is much smaller and we split it into $25$, $5$ and $10$ images for training, validation, and testing. This leaves us with $10$ pairs for validation and $45$ for testing.

For inspiration-exhale registration of lung CT scans we use a recently published Lung250M-4B dataset by \citet{falta2024lung250m}. The dataset is based on multiple earlier datasets: DIR-LAB COPDgene \citep{castillo2013reference}, EMPIRE10 \citep{murphy2011evaluation}, L2R-LungCT dataset from Learn2reg challenge \citep{hering2022learn2reg}, The National Lung Screening Trial (NLST) dataset \citep{doi:10.1056/NEJMoa1102873, nlst2013, TCIAClark2013}, TCIA-Ventilation dataset \citep{TCIAClark2013, ESLICK2018267, Eslick2022-mu}, and TCIA-NSCLC dataset \citep{TCIAClark2013, Hugo2016-aa, hugoetal2017, BALIK2013372, ROMAN20121566}. In total the dataset has 97, 18, and 9 pairs for training, testing, and validation. However, we can not use 12 image pairs from EMPIRE10 dataset due terms of use, 10 of which are part of the training set, and 2 of which are part of the validation set. The test set is not affected. For evaluation the data set includes manually placed landmarks on the validation and the test image pairs. In addition, the dataset contains automatically generated landmarks for training pairs but we want to stay in the unsupervised setting and do not use those for training. We use lung masked images for training, and downsample them to isotropic $2$ mm resolution, normalize and clip them to range $[-1024, 0]$, and crop the volumes to the shape of $160 \times 112 \times 160$.


%
%

\subsection{Evaluation metrics}

We evaluate brain subject-to-subject \textit{registration accuracy} using segmentations of brain structures included in the datasets: ($35$ structures for OASIS and $56$ for LPBA40), and two metrics: Dice score (Dice) and 95\% quantile of the Hausdorff distances (HD95), similarly to  Learn2Reg challenge \citep{hering2022learn2reg}. Dice score measures the overlap of the segmentations of source images deformed by the method and the segmentations of target images, and HD95 measures the distance between the surfaces of the segmentations.

For evaluating the inspiration-exhale registration of lung CT scans we use mean distance between landmark pairs after registration, denoted as target registration error (TRE).
However, comparing methods only based on the overlap of anatomic regions or landmarks is insufficient \citep{pluim2016truth, rohlfing2011image}, and hence also \textit{deformation regularity} should be measured, for which we use conventional metrics based on the local Jacobian determinants at $10^6$ sampled locations in each volume. The local derivatives were estimated via small perturbations of $10^{-7}$ voxels. We measure topology preservation as the proportion of the locations with a negative determinant ($\%$ of $|J_{\phi}|_{\leq 0}$), and deformation smoothness as the standard deviation of the determinant ($\operatorname{std}(|J_{\phi}|)$). Additionally we report inverse and cycle \textit{consistency} errors, see Section \ref{sec:intro}.

\subsection{Baseline methods}

We compare against \textit{VoxelMorph} \citep{balakrishnan2019voxelmorph}, \textit{SYMNet} \citep{mok2020fast}, conditional LapIRN (\textit{cLapIRN}) \citep{mok2020large, mok2021conditional}, and \textit{ByConstructionICON} \citep{greer2023inverse}. VoxelMorph is a standard baseline in deep learning based unsupervised registration. With SYMNet we are interested in how well our method preserves topology and how accurate the generated inverse deformations are compared to the SVF based methods. Additionally, since SYMNet is symmetric from the loss point of view, it is interesting to see how symmetric predictions it produces in practice. cLapIRN was the best method on OASIS dataset in Learn2Reg 2021 challenge \citep{hering2022learn2reg}. ByConstructionICON is a parallel work to ours developing a multi-step inverse consistent, symmetric, and topology preserving deep learning registration method based on SVF formulation. We used the official implementations\footnote{\href{https://github.com/voxelmorph/voxelmorph}{https://github.com/voxelmorph/voxelmorph}}\footnote{\href{https://github.com/cwmok/Fast-Symmetric-Diffeomorphic-Image-Registration-with-Convolutional-Neural-Networks}{https://github.com/cwmok/Fast-Symmetric-Diffeomorphic-Image-Registration-with-Convolutional-Neural-Networks}}\footnote{\href{https://github.com/cwmok/Conditional_LapIRN}{https://github.com/cwmok/Conditional\_LapIRN/}}\footnote{\href{https://github.com/uncbiag/ByConstructionICON}{https://github.com/uncbiag/ByConstructionICON/}} adjusted to our datasets. SYMNet uses anti-folding loss to penalize negative determinant. Since this loss is a separate component that could be easily used with any method, we also train SYMNet without it, denoted \textit{SYMNet (simple)} for two of our four datasets. This provides a comparison on how well the vanilla SVF framework can generate invertible deformations in comparison to our method. For details on hyperparameter selection for baseline models, see Appendix \ref*{appendix:baseline_hyperparameter}.

\subsection{Ablation study}

To study the usefulness of the symmetric formulation introduced in Section \ref{sec:symmetric_formulation_introduction}, we also conduct the experiments without it while keeping the architecture otherwise identical. In more detail, we change Equation \ref{eq:update_deformation_formula} into the following form:
\begin{equation}
    \delta^{(k)} := u^{(k)}(z_1^{(k)}, z_2^{(k)}) \circ u^{(k)}(z_1^{(k)}, z_2^{(k)}).
\end{equation}
For the inverse update (not necessarily inverse anymore despite the notation) we then simply use
\begin{equation}
    \left(\delta^{(k)}\right)^{-1} := u^{(k)}(z_2^{(k)}, z_1^{(k)}) \circ u^{(k)}(z_2^{(k)}, z_1^{(k)}).
\end{equation}
The resulting architecture is still topology preserving but no longer inverse or cycle consistent by construction. We refer to the architecture as \textit{SITReg (non-sym)}

\section{Results}\label{sec:results}

Evaluation results for the affinely pre-aligned OASIS dataset are shown in Table \ref{table:results_oasis}, the OASIS raw dataset in Table \ref{table:results_oasis_raw}., the LPBA40 dataset in Table \ref{table:results_lpba40}, and for the Lung250M-4B dataset in Table \ref{table:results_lung250m_4b}. Figure \ref{fig:deformation_regularity_visualization} visualizes differences in deformation regularity on OASIS dataset, and additional visualizations are available in Appendix \ref{appendix:result_visualization}. A comparison of the methods' inference time efficiencies on OASIS dataset are shown in Table \ref{table:performance_results}.

\begin{table*}[h]
\centering
\setlength\tabcolsep{1.5pt}
\scriptsize
\caption{\textbf{Results, OASIS dataset.} Mean and standard deviation of each metric are computed on the test set.
The percentage of folding voxels ($\%$ of $|J_{\phi}|_{\leq 0}$) from the complete SITReg version is shown in {\color{blue}blue}, other results are with the standard version (see Section \ref{sec:inference}).
VoxelMorph and cLapIRN do not predict inverse deformations and hence the inverse-consistency error is not shown.}
\begin{tabular}{lcccccc}

\toprule
& \multicolumn{2}{c}{Accuracy} & \multicolumn{2}{c}{Deformation regularity} & \multicolumn{2}{c}{Consistency}\\
\cmidrule(r){2-3} \cmidrule(r){4-5} \cmidrule(r){6-7}
Model & Dice $\uparrow$ & HD95 $\downarrow$ & $\%$ of $|J_{\phi}|_{\leq 0} \downarrow$ & $\operatorname{std}(|J_{\phi}|) \downarrow$ & Cycle $\downarrow$ & Inverse $\downarrow$ \\
\midrule SYMNet (original) & $0.788 (0.029)$ &   $2.15 (0.57)$   & $\bm{\expnumber{1.5}{-3}} (\expnumber{4.3}{-4})$ &               $\bm{0.44} (0.039)$                & $\expnumber{3.0}{-1} (\expnumber{2.9}{-2})$ & $\bm{\expnumber{3.5}{-3}} (\expnumber{4.2}{-4})$ \\
 SYMNet (simple)            & $0.787 (0.029)$ &   $2.17 (0.58)$   & $\expnumber{1.5}{-2} (\expnumber{3.1}{-3})$ &               $0.46 (0.045)$                & $\expnumber{2.8}{-1} (\expnumber{2.8}{-2})$ & $\expnumber{5.2}{-3} (\expnumber{8.4}{-4})$ \\
 VoxelMorph                 & $0.803 (0.031)$ &   $2.08 (0.57)$   & $\expnumber{1.4}{-1} (\expnumber{9.4}{-2})$ &               $0.49 (0.032)$                & $\expnumber{4.5}{-1} (\expnumber{5.3}{-2})$ &                      -                      \\
 cLapIRN                    & $0.812 (0.027)$ &   $1.93 (0.50)$   & $\expnumber{1.1}{0} (\expnumber{2.1}{-1})$  &               $0.55 (0.032)$                & $\expnumber{1.2}{0} (\expnumber{1.6}{-1})$  &                      -                      \\
 ByConstructionICON         & $0.813 (0.022)$ &   $1.83 (0.42)$   & $\expnumber{2.3}{-2} (\expnumber{6.4}{-3})$ &               $0.48 (0.069)$                & $\bm{\expnumber{5.3}{-3}} (\expnumber{1.5}{-3})$ &  $\expnumber{5.3}{-3} (\expnumber{1.5}{-3})$                     \\
 \midrule SITReg            & $0.818 (0.025)^*$ &   $1.84 (0.45)$   & $\expnumber{8.1}{-3} (\expnumber{1.6}{-3}) / \color{blue}{\bm{0}(0)}$ &               $0.45 (0.038)$                & $\expnumber{5.5}{-3} (\expnumber{6.9}{-4})$ & $\expnumber{5.5}{-3} (\expnumber{6.9}{-4})$ \\
 SITReg (non-sym)           & $\bm{0.819} (0.024)^*$ &   $\bm{1.82} (0.44)$   & $\expnumber{1.7}{-2} (\expnumber{4.5}{-3})$ &               $0.50 (0.037)$                & $\expnumber{1.2}{-1} (\expnumber{1.4}{-2})$ & $\expnumber{1.2}{-1} (\expnumber{1.4}{-2})$ \\

\midrule
\multicolumn{7}{l}{
$^*$ Statistically significant ($p < 0.05$) improvement compared to the baselines, for details see Appendix \ref{appendix:statistical_significance}.
}

\end{tabular}
\label{table:results_oasis}
\end{table*}

\begin{table*}[h]
\centering
\setlength\tabcolsep{1.5pt}
\scriptsize
\caption{\textbf{Results, OASIS raw dataset.} The results are interpreted similarly to Table \ref{table:results_oasis}. SYMNet and VoxelMorph did not converge to anything meaningful due to the large initial displacement.}
\begin{tabular}{lcccccc}

\toprule
& \multicolumn{2}{c}{Accuracy} & \multicolumn{2}{c}{Deformation regularity} & \multicolumn{2}{c}{Consistency}\\
\cmidrule(r){2-3} \cmidrule(r){4-5} \cmidrule(r){6-7}
Model & Dice $\uparrow$ & HD95 $\downarrow$ & $\%$ of $|J_{\phi}|_{\leq 0} \downarrow$ & $\operatorname{std}(|J_{\phi}|) \downarrow$ & Cycle $\downarrow$ & Inverse $\downarrow$ \\
\midrule SYMNet    & $0.176 (0.14)$  &   $19.7 (10.)$    & $\bm{\expnumber{1.8}{-4}} (\expnumber{1.8}{-4})$ &               $0.24 (0.031)$                & $\expnumber{3.6}{-1} (\expnumber{1.4}{-1})$ & $\expnumber{7.2}{-4} (\expnumber{2.1}{-4})$ \\
 VoxelMorph         & $0.230 (0.19)$  &   $19.4 (11.)$    & $\expnumber{9.2}{-2} (\expnumber{4.5}{-2})$ &               $0.26 (0.019)$                &  $\expnumber{4.0}{0} (\expnumber{2.3}{0})$  &                      -                      \\
 cLapIRN            & $0.744 (0.073)$ &   $3.14 (1.9)$    & $\expnumber{5.8}{-1} (\expnumber{1.4}{-1})$ &               $0.38 (0.045)$                &  $\expnumber{3.0}{0} (\expnumber{2.1}{0})$  &                      -                      \\
 ByConstructionICON & $0.803 (0.023)$ &   $1.83 (0.55)$   & $\expnumber{3.3}{-3} (\expnumber{2.0}{-3})$ &               $0.21 (0.045)$                & $\bm{\expnumber{1.1}{-3}} (\expnumber{6.1}{-4})$ &  $\bm{\expnumber{1.1}{-3}} (\expnumber{6.1}{-4})$              \\
 \midrule SITReg    & $\bm{0.813} (0.023)^*$ &   $\bm{1.80} (0.52)^*$   & $\expnumber{1.0}{-3} (\expnumber{3.9}{-4})/ \color{blue}{\bm{0}(0)}$ &               $\bm{0.20} (0.031)$                & $\expnumber{1.3}{-3} (\expnumber{3.2}{-4})$ & $\expnumber{1.3}{-3} (\expnumber{3.2}{-4})$ \\

\midrule
\multicolumn{7}{l}{
$^*$ Statistically significant ($p < 0.05$) improvement compared to the baselines, for details see Appendix \ref{appendix:statistical_significance}.
}

\end{tabular}
\label{table:results_oasis_raw}
\end{table*}

\begin{table*}[h]
\centering
\setlength\tabcolsep{1.5pt}
\scriptsize
\caption{\textbf{Results, LPBA40 dataset.} The results are interpreted similarly to Table \ref{table:results_oasis}.}
\begin{tabular}{lcccccc}

\toprule
& \multicolumn{2}{c}{Accuracy} & \multicolumn{2}{c}{Deformation regularity} & \multicolumn{2}{c}{Consistency}\\
\cmidrule(r){2-3} \cmidrule(r){4-5} \cmidrule(r){6-7}
Model & Dice $\uparrow$ & HD95 $\downarrow$ & $\%$ of $|J_{\phi}|_{\leq 0} \downarrow$ & $\operatorname{std}(|J_{\phi}|) \downarrow$ & Cycle $\downarrow$ & Inverse $\downarrow$ \\
 \midrule SYMNet (original) & $0.669 (0.033)$ &   $6.79 (0.70)$   & $\bm{\expnumber{1.1}{-3}} (\expnumber{4.6}{-4})$ &               $0.35 (0.050)$                & $\expnumber{2.7}{-1} (\expnumber{6.1}{-2})$ & $\expnumber{2.1}{-3} (\expnumber{4.3}{-4})$ \\
 SYMNet (simple)            & $0.664 (0.034)$ &   $6.88 (0.73)$   & $\expnumber{4.7}{-3} (\expnumber{1.6}{-3})$ &               $0.37 (0.053)$                & $\expnumber{2.8}{-1} (\expnumber{5.8}{-2})$ & $\expnumber{2.9}{-3} (\expnumber{6.7}{-4})$ \\
 VoxelMorph                 & $0.676 (0.032)$ &   $6.72 (0.68)$   & $\expnumber{2.2}{-1} (\expnumber{2.1}{-1})$ &               $0.35 (0.040)$                & $\expnumber{3.1}{-1} (\expnumber{1.1}{-1})$ & -                                           \\
 cLapIRN                    & $0.714 (0.019)$ &   $5.93 (0.43)$   & $\expnumber{8.4}{-2} (\expnumber{2.9}{-2})$ &               $\bm{0.27} (0.020)$                & $\expnumber{5.6}{-1} (\expnumber{1.8}{-1})$ & -                                           \\
 ByConstructionICON         & $0.674 (0.031)$ &   $6.60 (0.71)$   & $\expnumber{4.7}{-3} (\expnumber{2.9}{-3})$ &                $0.33 (0.41)$                & $\bm{\expnumber{1.6}{-3}} (\expnumber{6.6}{-4})$ & $\bm{\expnumber{1.6}{-3}} (\expnumber{6.6}{-4})$         \\
 \midrule SITReg            & $\bm{0.720} (0.017)^*$ &   $\bm{5.88} (0.43)$   & $\expnumber{2.4}{-3} (\expnumber{6.4}{-4}) / \color{blue}{\bm{0}(0)}$ &               $0.31 (0.032)$                & $\expnumber{2.6}{-3} (\expnumber{4.2}{-4})$ & $\expnumber{2.6}{-3} (\expnumber{4.2}{-4})$ \\
 SITReg (non-sym)           & $0.697 (0.024)$ &   $6.29 (0.57)$   & $\expnumber{1.2}{-3} (\expnumber{5.6}{-4})$ &               $0.34 (0.033)$                & $\expnumber{2.2}{-1} (\expnumber{3.2}{-2})$ & $\expnumber{2.2}{-1} (\expnumber{3.2}{-2})$ \\
\midrule
\multicolumn{7}{l}{$^*$ Statistically significant ($p < 0.05$) improvement compared to the baselines, for details see Appendix \ref{appendix:statistical_significance}.}
\end{tabular}
\label{table:results_lpba40}
\end{table*}

\begin{table*}[h]
\centering
\setlength\tabcolsep{1.5pt}
\scriptsize
\caption{\textbf{Results, Lung250M-4B dataset.} The results are interpreted similarly to Table \ref{table:results_oasis}. Note that TRE is computed in the inhale coordinates in contrast to the exhale coordinates used in the paper by \citet{falta2024lung250m}. In the inhale coordinates our method obtains TRE=$2.27$.}
\begin{tabular}{lccccc}

\toprule
& \multicolumn{1}{c}{} & \multicolumn{2}{c}{Deformation regularity} & \multicolumn{2}{c}{Consistency}\\
\cmidrule(r){3-4} \cmidrule(r){5-6}
Model & TRE $\downarrow$ & $\%$ of $|J_{\phi}|_{\leq 0} \downarrow$ & $\operatorname{std}(|J_{\phi}|) \downarrow$ & Cycle $\downarrow$ & Inverse $\downarrow$ \\
 \midrule SYMNet    &   $8.25 (2.1)$   &  $\bm{\expnumber{0.0}{0}} (\expnumber{0.0}{0})$  &               $0.32 (0.093)$                & $\expnumber{5.3}{-1} (\expnumber{2.2}{-1})$ & $\expnumber{3.6}{-4} (\expnumber{1.7}{-4})$ \\
 VoxelMorph         &   $6.66 (1.9)$   & $\expnumber{1.2}{0} (\expnumber{5.9}{-1})$  &               $0.39 (0.090)$                &  $\expnumber{1.1}{1} (\expnumber{6.5}{0})$  & -                                           \\
 cLapIRN            &   $5.34 (1.9)$   & $\expnumber{4.9}{-3} (\expnumber{6.5}{-3})$ &               $0.25 (0.071)$                &  $\expnumber{5.9}{0} (\expnumber{3.7}{0})$  & -                                           \\
 ByConstructionICON &   $8.63 (3.1)$   &  $\bm{\expnumber{0.0}{0}} (\expnumber{0.0}{0})$  &               $\bm{0.19} (0.054)$                & $\bm{\expnumber{4.8}{-5}} (\expnumber{2.8}{-5})$ & $\bm{\expnumber{4.8}{-5}} (\expnumber{2.8}{-5})$                                           \\
 \midrule SITReg    &  $\bm{2.71} (0.93)^*$   & $\expnumber{4.0}{-5} (\expnumber{6.6}{-5}) / \color{blue}{\bm{0}(0)}$ &               $0.30 (0.084)$                & $\expnumber{1.2}{-3} (\expnumber{4.5}{-4})$ & $\expnumber{1.2}{-3} (\expnumber{4.5}{-4})$ \\
 SITReg (non-sym)   &   $4.98 (2.3)$   &  $\bm{\expnumber{0.0}{0}} (\expnumber{0.0}{0})$  &               $0.32 (0.088)$                & $\expnumber{4.6}{-1} (\expnumber{2.8}{-1})$ & $\expnumber{4.6}{-1} (\expnumber{2.8}{-1})$ \\
\midrule
\multicolumn{6}{l}{$^*$ Statistically significant ($p < 0.05$) improvement compared to the baselines, for details see Appendix \ref{appendix:statistical_significance}.}
\end{tabular}
\label{table:results_lung250m_4b}
\end{table*}

\begin{table*}[h]
\centering
\setlength\tabcolsep{3pt}
\scriptsize
\caption{\textbf{Computational efficiency, OASIS dataset.} Mean and standard deviation are shown. Inference time and memory usage were measured on NVIDIA GeForce RTX 3090.}
\begin{tabular}{lccc}
\toprule
 Model                               & Inference Time (s) $\downarrow$ & Inference Memory (GB) $\downarrow$ & \# parameters (M) $\downarrow$ \\
 \midrule SYMNet    &        $0.095 (0.00052)$        &               $\bm{1.9}$                &             $\bm{0.9}$              \\
 VoxelMorph         &         $0.16 (0.0010)$         &               $5.6$                &             $1.3$              \\
 cLapIRN            &        $\bm{0.10} (0.00052)$         &               $4.1$                &             $1.2$              \\
 ByConstrictionICON &         $0.32 (0.0022)$         &               $2.4$                &             $45.1$             \\
 \midrule SITReg    &         $0.37 (0.0057)$         &               $3.4$                &             $1.2$              \\
\bottomrule
\end{tabular}
\label{table:performance_results}
\end{table*}
\clearpage
\begin{figure}[h]
\centering
\includegraphics[width=1.0\textwidth]{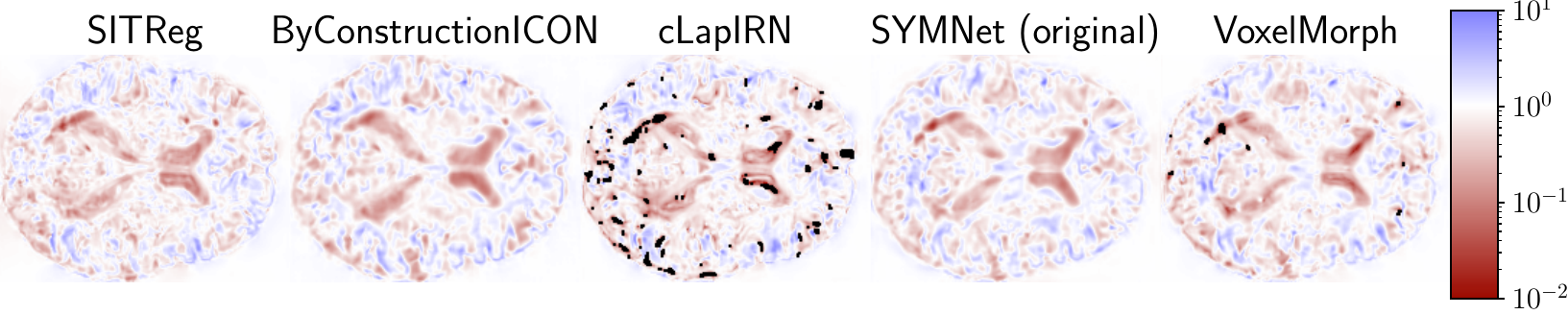}
\caption{\textbf{Visual deformation regularity comparison.} Local Jacobian determinants are visualized for each model for a single predicted deformation in OASIS experiment. Folding voxels (determinant below zero) are marked with black color. Only one axial slice of the predicted 3D deformation is visible.}
\label{fig:deformation_regularity_visualization}
\end{figure}

\section{Discussion}\label{sec:discussion}

Evaluation of registration algorithms is difficult due to lack of ground truth \citep{pluim2016truth}. In particular, the assessment of registration performance should not be based only on a tissue overlap score since a clearly unrealistic and wrong deformation could still have a good tissue overlap score (an extreme example is presented by \citet{rohlfing2011image}). For that reason, usually also deformation regularity is measured. However, that introduces further difficulties for evaluation because of the trade-off invoked by the regularization hyperparameter (in our case $\lambda$) between the tissue overlap and deformation regularity, which can change the ranking of the methods for different metrics. For that reason one should look at the overall performance across the regularity and accuracy metrics (see e.g. Learn2reg challenge \citep{hering2022learn2reg}). We further evaluate using the consistency metrics as is often done in the literature \citep{holden2000voxel, pluim2016truth}.

In a such overall comparison over all four datasets, our method performs clearly the best. In more detail:

\begin{itemize}
    \item \textbf{VoxelMorph:} Our method outperforms it on every metric in every experiment.
    \item \textbf{SYMNet:} While SYMNet (original) has in general slightly better deformation regularity (compared to our standard inference variant), our method has a significantly better dice score or TRE. By increasing regularization one could make our model to have better regularity while still maintaining significantly better dice score or TRE than SYMNet. This is demonstrated for validation set results in Tables \ref{appendix-table:hyperparameter_optimization_oasis}, \ref{appendix-table:hyperparameter_optimization_lpba40}, and \ref{appendix-table:hyperparameter_optimization_lung250m_4b} in Appendix \ref{appendix:hyperparameter_selection}. In other words, our method has significantly better overall performance.
    \item \textbf{cLapIRN:} Our method outperforms cLapIRN on all four datasets. On OASIS raw and Lung250M-4B datasets our method outperforms cLapIRN clearly. On the pre-aligned OASIS and LPBA40 datasets our method has only slightly better tissue overlap performance than cLapIRN, but has clearly superior deformation regularity in terms of folding voxels.
    \item \textbf{ByConstructionICON:} On the OASIS datasets ByCounstrictionICON performs similarily to our method, although slightly worse on the OASIS raw dataset. However, on the smaller LPBA40 and Lung250M-4B datasets the performance of our method is clearly better, suggesting that our method generalizes from less data than ByConstructionICON.
\end{itemize}

For the Lung250M-4B dataset the test set is standardized, and we can compare the results to the methods benchmarked in the paper \citep{falta2024lung250m}. However, our metrics are computed in the coordinates of the inhale image whereas the paper uses coordinates of the exhale image. In the exhale coordinates our method obtains TRE $=2.27$ which is very competitive, clearly outperforming the best deep learning method VoxelMorph++ without instance optimization (TRE $=4.47$) trained using additional landmark supervision unlike our unsupervised method. VoxelMorph++ with instance optimization obtains TRE $=2.26$.

\subsection{Ablation study}

Based on the ablation study, the symmetric formulation helps especially on the smaller LPBA40 and Lung250M-4B datasets. On the larger OASIS dataset the performance on tissue overlap metrics is similar but even there is a slight improvement in terms of deformation regularity. The result is in line with the common machine learning wisdom that incorporating inductive bias into models has more relevance when the training set is small.

\subsection{Computational performance}

Inference time of our method is slightly larger than that of the compared methods, but unlike VoxelMorph and cLapIRN, it produces deformations in both directions immediately. Also, half a second runtime is still very fast and restrictive only in the most time-critical use cases. In terms of inference memory usage our method is competitive.

\section{Conclusions}

We proposed a novel image registration architecture inbuilt with the desirable properties of symmetry, inverse consistency, and topology preservation. The multi-resolution formulation was capable of accurately registering images even with large intial misalignments. As part of our method, we developed a new neural network component \textit{deformation inversion layer}. The model is easily end-to-end trainable and does not require tedious multi-stage training strategies. In the experiments the method demonstrates state-of-the-art registration performance. The main limitation is somewhat heavier computational cost than other methods.


\acks{This work was supported by the Academy of Finland (Flagship programme: Finnish Center for Artificial Intelligence FCAI, and grants 336033, 352986) and EU (H2020 grant 101016775 and NextGenerationEU). We also acknowledge the computational resources provided by the Aalto Science-IT Project.}

%
\ethics{The work follows appropriate ethical standards in conducting research and writing the manuscript, following all applicable laws and regulations regarding treatment of animals or human subjects.}

\coi{We declare we don't have conflicts of interest.}

\data{All data used in the experiments is publicly available. The exact OASIS dataset used can be accessed from the website of Learn2reg challenge \url{https://learn2reg.grand-challenge.org/} and the LPBA40 dataset is available for download at \url{https://www.loni.usc.edu/research/atlas_downloads}. Lung250M-4B dataset can be obtained based on instructions at \url{https://github.com/multimodallearning/Lung250M-4B}. The provided codebase contains implementation for all additional data pre-processing that was done.}
\clearpage
\bibliography{references}


\clearpage
\renewcommand{\theHsection}{A\arabic{section}}
\appendix

\section{Hyperparameter selection details}\label{appendix:hyperparameter_selection}

We experimented on validation set with different hyperparameters during the development. While the final results on test set are computed only for one chosen configuration, the results on validation set might still be of interest for the reader. Results of these experiments for the pre-aligned OASIS dataset are shown in Table \ref{appendix-table:hyperparameter_optimization_oasis}, for the LPBA40 dataset in Table \ref{appendix-table:hyperparameter_optimization_lpba40}, and for the Lung250M-4B dataset in Table \ref{appendix-table:hyperparameter_optimization_lung250m_4b}.

For the OASIS raw dataset without pre-alignment we used $6$ resolution levels, together with an affine transformation prediction stage before the other deformation updates. We omitted the predicted affine transformation from the deformation regularization. The same regularization weight was used as for the pre-aligned OASIS dataset.

\begin{table}[h]
\centering
\setlength\tabcolsep{3pt}
\scriptsize
\caption{Hyperparameter optimization results for our method calculated on the OASIS validation set. The chosen configuration was $\lambda = 1.0$, and $K=4$. HD95 metric is not included due to relatively high computational cost.}
\begin{tabular}{cccccccc}
\toprule
\multicolumn{2}{c}{Hyperparameters}  & \multicolumn{1}{c}{Accuracy} & \multicolumn{2}{c}{Deformation regularity} & \multicolumn{2}{c}{Consistency}\\
\cmidrule(r){1-2} \cmidrule(r){3-3} \cmidrule(r){4-5} \cmidrule(r){6-7}
$\lambda$  & $K$ & Dice $\uparrow$ & $\%$ of $|J_{\phi}|_{\leq 0} \downarrow$ & $\operatorname{std}(|J_{\phi}|) \downarrow$ &       Cycle $\downarrow$        &      Inverse $\downarrow$       \\
 \midrule 1.0    &  5  & $0.822 (0.035)$ & $\expnumber{9.1}{-3} (\expnumber{1.7}{-3})$ &               $0.45 (0.027)$                & $\expnumber{5.7}{-3} (\expnumber{6.0}{-4})$ & $\expnumber{5.7}{-3} (\expnumber{6.0}{-4})$ \\
    1.5    &  5  & $0.818 (0.034)$ & $\expnumber{1.9}{-3} (\expnumber{5.1}{-4})$ &               $0.40 (0.023)$                & $\expnumber{3.7}{-3} (\expnumber{3.4}{-4})$ & $\expnumber{3.7}{-3} (\expnumber{3.4}{-4})$ \\
    2.0    &  5  & $0.815 (0.035)$ & $\expnumber{3.7}{-4} (\expnumber{2.0}{-4})$ &               $0.37 (0.021)$                & $\expnumber{2.6}{-3} (\expnumber{2.1}{-4})$ & $\expnumber{2.6}{-3} (\expnumber{2.1}{-4})$ \\
    1.0    &  4  & $0.822 (0.034)$ & $\expnumber{8.2}{-3} (\expnumber{1.5}{-3})$ &               $0.44 (0.028)$                & $\expnumber{5.5}{-3} (\expnumber{5.6}{-4})$ & $\expnumber{5.5}{-3} (\expnumber{5.6}{-4})$ \\
    1.5    &  4  & $0.819 (0.035)$ & $\expnumber{2.1}{-3} (\expnumber{5.8}{-4})$ &               $0.40 (0.023)$                & $\expnumber{3.4}{-3} (\expnumber{3.3}{-4})$ & $\expnumber{3.4}{-3} (\expnumber{3.3}{-4})$ \\
    2.0    &  4  & $0.815 (0.036)$ & $\expnumber{3.6}{-4} (\expnumber{2.1}{-4})$ &               $0.37 (0.020)$                & $\expnumber{2.6}{-3} (\expnumber{2.2}{-4})$ & $\expnumber{2.6}{-3} (\expnumber{2.2}{-4})$ \\
\bottomrule
\end{tabular}
\label{appendix-table:hyperparameter_optimization_oasis}
\end{table}

\begin{table}[h]
\centering
\setlength\tabcolsep{3pt}
\scriptsize
\vskip -1em
\caption{Hyperparameter optimization results for our method calculated on the LPBA40 validation set. The chosen configuration was $\lambda = 1.0$, $K=7$, and $\text{Affine} = \text{No}$.}
\begin{tabular}{ccccccccccc}
\toprule
\multicolumn{3}{c}{Hyperparameters}  & \multicolumn{2}{c}{Accuracy} & \multicolumn{2}{c}{Deformation regularity} & \multicolumn{2}{c}{Consistency}\\
\cmidrule(r){1-3} \cmidrule(r){4-5} \cmidrule(r){6-7} \cmidrule(r){8-9}
$\lambda$  & $K$ & Affine & Dice $\uparrow$ & HD95 $\downarrow$ & $\%$ of $|J_{\phi}|_{\leq 0} \downarrow$ & $\operatorname{std}(|J_{\phi}|) \downarrow$ &       Cycle $\downarrow$        &      Inverse $\downarrow$       \\
 \midrule 1.0 &  4  &   No   & $0.710 (0.015)$ &   $6.10 (0.46)$   & $\expnumber{2.5}{-3} (\expnumber{8.7}{-4})$ &               $0.31 (0.020)$                & $\expnumber{2.5}{-3} (\expnumber{3.5}{-4})$ & $\expnumber{2.5}{-3} (\expnumber{3.5}{-4})$ \\
    1.0     &  5  &   No   & $0.720 (0.014)$ &   $5.83 (0.36)$   & $\expnumber{1.7}{-3} (\expnumber{5.7}{-4})$ &               $0.30 (0.019)$                & $\expnumber{2.3}{-3} (\expnumber{3.1}{-4})$ & $\expnumber{2.3}{-3} (\expnumber{3.1}{-4})$ \\
    1.0     &  6  &   No   & $0.725 (0.012)$ &   $5.70 (0.31)$   & $\expnumber{2.1}{-3} (\expnumber{4.8}{-4})$ &               $0.29 (0.019)$                & $\expnumber{2.3}{-3} (\expnumber{3.0}{-4})$ & $\expnumber{2.3}{-3} (\expnumber{3.0}{-4})$ \\
    1.0     &  7  &   No   & $0.726 (0.011)$ &   $5.69 (0.30)$   & $\expnumber{1.9}{-3} (\expnumber{5.3}{-4})$ &               $0.29 (0.019)$                & $\expnumber{2.3}{-3} (\expnumber{3.1}{-4})$ & $\expnumber{2.3}{-3} (\expnumber{3.1}{-4})$ \\
    1.0     &  5  &  Yes   & $0.719 (0.014)$ &   $5.86 (0.35)$   & $\expnumber{2.2}{-3} (\expnumber{7.2}{-4})$ &               $0.30 (0.019)$                & $\expnumber{2.4}{-3} (\expnumber{3.3}{-4})$ & $\expnumber{2.4}{-3} (\expnumber{3.3}{-4})$ \\
    1.0     &  6  &  Yes   & $0.721 (0.015)$ &   $5.78 (0.37)$   & $\expnumber{2.6}{-3} (\expnumber{5.8}{-4})$ &               $0.30 (0.019)$                & $\expnumber{2.4}{-3} (\expnumber{3.2}{-4})$ & $\expnumber{2.4}{-3} (\expnumber{3.2}{-4})$ \\
    2.0     &  4  &   No   & $0.703 (0.018)$ &   $6.20 (0.50)$   & $\expnumber{5.0}{-5} (\expnumber{8.1}{-5})$ &               $0.25 (0.017)$                & $\expnumber{1.2}{-3} (\expnumber{1.5}{-4})$ & $\expnumber{1.2}{-3} (\expnumber{1.5}{-4})$ \\
    2.0     &  5  &   No   & $0.718 (0.014)$ &   $5.84 (0.35)$   & $\expnumber{6.5}{-5} (\expnumber{7.9}{-5})$ &               $0.25 (0.016)$                & $\expnumber{1.1}{-3} (\expnumber{1.5}{-4})$ & $\expnumber{1.1}{-3} (\expnumber{1.5}{-4})$ \\
    2.0     &  6  &   No   & $0.722 (0.012)$ &   $5.76 (0.30)$   & $\expnumber{5.0}{-5} (\expnumber{7.4}{-5})$ &               $0.24 (0.016)$                & $\expnumber{1.1}{-3} (\expnumber{1.5}{-4})$ & $\expnumber{1.1}{-3} (\expnumber{1.5}{-4})$ \\
    2.0     &  7  &   No   & $0.721 (0.012)$ &   $5.77 (0.30)$   & $\expnumber{4.0}{-5} (\expnumber{5.8}{-5})$ &               $0.25 (0.016)$                & $\expnumber{1.2}{-3} (\expnumber{1.5}{-4})$ & $\expnumber{1.2}{-3} (\expnumber{1.5}{-4})$ \\
    2.0     &  5  &  Yes   & $0.715 (0.014)$ &   $5.95 (0.40)$   & $\expnumber{4.0}{-5} (\expnumber{8.6}{-5})$ &               $0.24 (0.016)$                & $\expnumber{1.1}{-3} (\expnumber{1.4}{-4})$ & $\expnumber{1.1}{-3} (\expnumber{1.4}{-4})$ \\
    2.0     &  6  &  Yes   & $0.721 (0.014)$ &   $5.77 (0.34)$   & $\expnumber{5.0}{-5} (\expnumber{7.4}{-5})$ &               $0.25 (0.017)$                & $\expnumber{1.1}{-3} (\expnumber{1.4}{-4})$ & $\expnumber{1.1}{-3} (\expnumber{1.4}{-4})$ \\
\bottomrule
\end{tabular}
\label{appendix-table:hyperparameter_optimization_lpba40}
\end{table}

\clearpage
\begin{table}[t]
\centering
\setlength\tabcolsep{3pt}
\vskip -1em
\scriptsize
\caption{Hyperparameter optimization results for our method calculated on the Lung250M-4b validation set. The chosen configuration was $\lambda = 1.0$, $K=6$.}
\begin{tabular}{ccccccccccc}
\toprule
\multicolumn{2}{c}{Hyperparameters}  & \multicolumn{1}{c}{Accuracy} & \multicolumn{2}{c}{Deformation regularity} & \multicolumn{2}{c}{Consistency}\\
\cmidrule(r){1-2}  \cmidrule(r){3-3}  \cmidrule(r){4-5} \cmidrule(r){6-7}
$\lambda$  & $K$ & TRE $\downarrow$ & $\%$ of $|J_{\phi}|_{\leq 0} \downarrow$ & $\operatorname{std}(|J_{\phi}|) \downarrow$ &       Cycle $\downarrow$        &      Inverse $\downarrow$       \\
 \midrule 1.0    &  5  &  $3.05 (0.84)$   & $\expnumber{1.2}{-5} (\expnumber{4.1}{-5})$ &                $0.21 (0.11)$                & $\expnumber{1.1}{-2} (\expnumber{4.1}{-3})$ & $\expnumber{1.1}{-2} (\expnumber{4.1}{-3})$ \\
    1.0    &  6  &  $2.77 (0.65)$   & $\expnumber{6.2}{-6} (\expnumber{2.4}{-5})$ &                $0.20 (0.11)$                & $\expnumber{9.0}{-3} (\expnumber{3.1}{-3})$ & $\expnumber{9.0}{-3} (\expnumber{3.1}{-3})$ \\
    2.0    &  6  &  $2.84 (0.68)$   &  $\expnumber{0.0}{0} (\expnumber{0.0}{0})$  &               $0.18 (0.088)$                & $\expnumber{7.0}{-3} (\expnumber{1.9}{-3})$ & $\expnumber{7.0}{-3} (\expnumber{1.9}{-3})$ \\
\bottomrule
\end{tabular}
\label{appendix-table:hyperparameter_optimization_lung250m_4b}
\end{table}
\section{Hyperparameter selection details for baselines}\label{appendix:baseline_hyperparameter}

For cLapIRN baseline we used the regularization parameter value $\overline{\lambda} = 0.05$ for the OASIS datasets, value $\overline{\lambda} = 0.1$ for the LPBA40 dataset, and $\overline{\lambda}=0.01$ for the Lung250M-4b dataset where $\overline{\lambda}$ is used as in the paper presenting the method \citep{mok2021conditional}. The values were chosen based on the validation set results shown in Tables \ref{appendix-table:clapirn_hyperparameter_optimization_oasis}, \ref{appendix-table:clapirn_hyperparameter_optimization_oasis_raw}, \ref{appendix-table:clapirn_hyperparameter_optimization_lpba40}, and \ref{appendix-table:clapirn_hyperparameter_optimization_lung250m_4b}.

For ByConstructionICON baseline we used regularization parameter $\lambda = 0.5$ for the OASIS datasets, $\lambda = 1.0$ for the LPBA40 dataset, and $\lambda = 5.0$ for the Lung250M-4B dataset. The values were chosen based on the validation set results shown in Tables \ref{appendix-table:byconstructionicon_hyperparameter_optimization_oasis}, \ref{appendix-table:byconstructionicon_hyperparameter_optimization_lpba40}, and \ref{appendix-table:byconstructionicon_hyperparameter_optimization_lung250m_4b}.

We trained VoxelMorph with losses and regularization weight identical to our method and for SYMNet we used hyperparameters directly provided by \citet{mok2020fast}. We used the default number of convolution features for the baselines except for VoxelMorph we doubled the number of features, as that was suggested in the paper \citep{balakrishnan2019voxelmorph}.

\begin{table}[h]
\centering
\setlength\tabcolsep{3pt}
\scriptsize
\caption{Regularization parameter optimization results for cLapIRN calculated on the pre-aligned OASIS validation set. Here $\overline{\lambda}$ refers to the normalized regularization weight of the gradient loss of cLapIRN and should be in range $[0,\ 1]$. Value $\overline{\lambda} = 0.05$ was chosen. HD95 metric is not included due to relatively high computational cost.}
\begin{tabular}{cccccc}
\toprule
Hyperparameters & \multicolumn{1}{c}{Accuracy} & \multicolumn{2}{c}{Deformation regularity} & Consistency\\
\cmidrule(r){1-1} \cmidrule(r){2-3} \cmidrule(r){4-5} \cmidrule(r){6-6}
$\overline{\lambda}$ & Dice $\uparrow$ & $|J_{\phi}|_{\leq 0} \downarrow$ & $\operatorname{std}(|J_{\phi}|) \downarrow$ &       Cycle $\downarrow$     \\
\midrule
     0.01      & $0.812 (0.034)$ & $\expnumber{2.5}{0} (\expnumber{2.9}{-1})$  &               $0.82 (0.048)$                & $\expnumber{1.7}{0} (\expnumber{1.5}{-1})$  \\
         0.05         & $0.817 (0.034)$ & $\expnumber{1.1}{0} (\expnumber{1.8}{-1})$  &               $0.56 (0.029)$                & $\expnumber{1.2}{0} (\expnumber{1.3}{-1})$  \\
         0.1          & $0.812 (0.035)$ & $\expnumber{4.2}{-1} (\expnumber{1.1}{-1})$ &               $0.43 (0.020)$                & $\expnumber{8.9}{-1} (\expnumber{1.1}{-1})$ \\
         0.2          & $0.798 (0.038)$ & $\expnumber{7.2}{-2} (\expnumber{3.9}{-2})$ &               $0.30 (0.013)$                & $\expnumber{6.0}{-1} (\expnumber{8.3}{-2})$ \\
         0.4          & $0.769 (0.042)$ & $\expnumber{1.4}{-3} (\expnumber{1.7}{-3})$ &               $0.18 (0.0087)$               & $\expnumber{3.5}{-1} (\expnumber{4.4}{-2})$ \\
         0.8          & $0.727 (0.049)$ & $\expnumber{3.4}{-6} (\expnumber{2.2}{-5})$ &               $0.10 (0.0050)$               & $\expnumber{2.5}{-1} (\expnumber{3.8}{-2})$ \\
         1.0          & $0.711 (0.052)$ & $\expnumber{1.3}{-6} (\expnumber{1.7}{-5})$ &              $0.082 (0.0042)$               & $\expnumber{2.3}{-1} (\expnumber{3.8}{-2})$ \\
\bottomrule
\end{tabular}

\label{appendix-table:clapirn_hyperparameter_optimization_oasis}
\end{table}

\begin{table}[h]
\centering
\setlength\tabcolsep{3pt}
\scriptsize
\caption{Regularization parameter optimization results for cLapIRN calculated on the OASIS raw validation set. The table is interpreted similarly to Table \ref{appendix-table:clapirn_hyperparameter_optimization_oasis}. Value $\overline{\lambda} = 0.05$ was chosen since it resulted in clearly the highest Dice score. HD95 metric is not included due to relatively high computational cost.}
\begin{tabular}{cccccc}
\toprule
Hyperparameters & \multicolumn{1}{c}{Accuracy} & \multicolumn{2}{c}{Deformation regularity} & Consistency\\
\cmidrule(r){1-1} \cmidrule(r){2-3} \cmidrule(r){4-5} \cmidrule(r){6-6}
$\overline{\lambda}$ & Dice $\uparrow$ & $|J_{\phi}|_{\leq 0} \downarrow$ & $\operatorname{std}(|J_{\phi}|) \downarrow$ &       Cycle $\downarrow$     \\
\midrule
     0.01      & $0.736 (0.11)$  & $\expnumber{4.9}{-1} (\expnumber{1.3}{-1})$ &               $0.36 (0.040)$                & $\expnumber{3.1}{0} (\expnumber{1.9}{0})$ \\
         0.02         & $0.738 (0.11)$  & $\expnumber{5.1}{-1} (\expnumber{1.3}{-1})$ &               $0.36 (0.038)$                & $\expnumber{3.2}{0} (\expnumber{2.2}{0})$ \\
         0.05         & $0.740 (0.11)$  & $\expnumber{2.9}{-1} (\expnumber{8.0}{-2})$ &               $0.28 (0.028)$                & $\expnumber{2.9}{0} (\expnumber{2.1}{0})$ \\
         0.1          & $0.733 (0.12)$  & $\expnumber{9.7}{-2} (\expnumber{3.4}{-2})$ &               $0.21 (0.019)$                & $\expnumber{2.6}{0} (\expnumber{2.1}{0})$ \\
\bottomrule
\end{tabular}

\label{appendix-table:clapirn_hyperparameter_optimization_oasis_raw}
\end{table}

\begin{table}[h]
\centering
\setlength\tabcolsep{3pt}
\scriptsize
\caption{Regularization parameter optimization results for cLapIRN calculated on the LPBA40 validation set. The table is interpreted similarly to Table \ref{appendix-table:clapirn_hyperparameter_optimization_oasis}. Value $\overline{\lambda} = 0.1$ was chosen due to the best overall performance.}
\begin{tabular}{cccccc}
\toprule
Hyperparameters & \multicolumn{2}{c}{Accuracy} & \multicolumn{2}{c}{Deformation regularity} & Consistency\\
\cmidrule(r){1-1} \cmidrule(r){2-3} \cmidrule(r){4-5} \cmidrule(r){6-6}
$\overline{\lambda}$ & Dice $\uparrow$ & HD95 $\downarrow$ & $|J_{\phi}|_{\leq 0} \downarrow$ & $\operatorname{std}(|J_{\phi}|) \downarrow$ &       Cycle $\downarrow$     \\
\midrule 0.01      & $0.714 (0.014)$ & $\expnumber{9.9}{-1} (\expnumber{1.5}{-1})$ &               $0.45 (0.029)$                &               $0.45 (0.029)$                & $\expnumber{9.9}{-1} (\expnumber{2.2}{-1})$ \\
         0.05         & $0.715 (0.014)$ & $\expnumber{3.2}{-1} (\expnumber{6.8}{-2})$ &               $0.33 (0.018)$                &               $0.33 (0.018)$                & $\expnumber{8.0}{-1} (\expnumber{2.1}{-1})$ \\
         0.1          & $0.714 (0.014)$ & $\expnumber{7.4}{-2} (\expnumber{2.4}{-2})$ &               $0.25 (0.012)$                &               $0.25 (0.012)$                & $\expnumber{6.6}{-1} (\expnumber{2.1}{-1})$ \\
         0.2          & $0.709 (0.015)$ & $\expnumber{4.4}{-3} (\expnumber{2.4}{-3})$ &               $0.19 (0.0090)$               &               $0.19 (0.0090)$               & $\expnumber{4.9}{-1} (\expnumber{1.9}{-1})$ \\
         0.4          & $0.698 (0.017)$ & $\expnumber{3.5}{-5} (\expnumber{5.7}{-5})$ &               $0.13 (0.0071)$               &               $0.13 (0.0071)$               & $\expnumber{3.6}{-1} (\expnumber{1.9}{-1})$ \\
         0.8          & $0.678 (0.019)$ & $\expnumber{5.0}{-6} (\expnumber{2.2}{-5})$ &              $0.085 (0.0062)$               &              $0.085 (0.0062)$               & $\expnumber{3.0}{-1} (\expnumber{1.9}{-1})$ \\
         1.0          & $0.671 (0.021)$ & $\expnumber{5.0}{-6} (\expnumber{2.2}{-5})$ &              $0.074 (0.0061)$               &              $0.074 (0.0061)$               & $\expnumber{3.0}{-1} (\expnumber{1.9}{-1})$ \\
\bottomrule
\end{tabular}
\label{appendix-table:clapirn_hyperparameter_optimization_lpba40}
\end{table}

\begin{table}[h]
\centering
\setlength\tabcolsep{3pt}
\scriptsize
\caption{Regularization parameter optimization results for cLapIRN calculated on the Lung250M-4B validation set. The table is interpreted similarly to Table \ref{appendix-table:clapirn_hyperparameter_optimization_oasis}. Value $\overline{\lambda} = 0.01$ was chosen due to the best overall performance.}
\begin{tabular}{cccccc}
\toprule
Hyperparameters & \multicolumn{1}{c}{Accuracy} & \multicolumn{2}{c}{Deformation regularity} & Consistency\\
\cmidrule(r){1-1}  \cmidrule(r){2-2} \cmidrule(r){3-4} \cmidrule(r){5-5}
$\overline{\lambda}$ & TRE $\downarrow$ & $|J_{\phi}|_{\leq 0} \downarrow$ & $\operatorname{std}(|J_{\phi}|) \downarrow$ &       Cycle $\downarrow$     \\
\midrule 0.01      &   $4.33 (1.5)$   & $\expnumber{1.7}{-3} (\expnumber{2.1}{-3})$ &               $0.20 (0.078)$                & $\expnumber{5.8}{0} (\expnumber{6.0}{0})$ \\
         0.05         &   $4.35 (1.6)$   & $\expnumber{1.4}{-3} (\expnumber{1.7}{-3})$ &               $0.19 (0.073)$                & $\expnumber{5.8}{0} (\expnumber{6.3}{0})$ \\
         0.1          &   $4.41 (1.7)$   & $\expnumber{8.9}{-4} (\expnumber{1.2}{-3})$ &               $0.18 (0.066)$                & $\expnumber{5.9}{0} (\expnumber{6.8}{0})$ \\
         0.2          &   $4.67 (2.1)$   & $\expnumber{3.9}{-4} (\expnumber{5.9}{-4})$ &               $0.17 (0.055)$                & $\expnumber{6.1}{0} (\expnumber{8.1}{0})$ \\
         0.4          &   $5.51 (3.6)$   & $\expnumber{1.5}{-4} (\expnumber{2.9}{-4})$ &               $0.14 (0.039)$                & $\expnumber{6.9}{0} (\expnumber{1.1}{1})$ \\
         0.8          &   $7.00 (5.7)$   & $\expnumber{1.1}{-4} (\expnumber{2.1}{-4})$ &               $0.13 (0.025)$                & $\expnumber{3.7}{0} (\expnumber{3.6}{0})$ \\
         1.0          &   $7.39 (5.9)$   & $\expnumber{7.3}{-5} (\expnumber{2.0}{-4})$ &               $0.12 (0.022)$                & $\expnumber{3.0}{0} (\expnumber{1.7}{0})$ \\
\bottomrule
\end{tabular}
\label{appendix-table:clapirn_hyperparameter_optimization_lung250m_4b}
\end{table}

\begin{table}[h]
\centering
\setlength\tabcolsep{3pt}
\scriptsize
\caption{Regularization parameter optimization results for ByConstructionICON calculated on the pre-aligned OASIS validation set. Here $\lambda$ refers to the weight of the regularizing bending energy loss used by the method. Value $\lambda = 0.5$ was chosen due to the best performance (the paper \citep{greer2023inverse} used $\lambda = 5.0$ for the OASIS dataset which we found to be suboptimal). HD95 metric is not included due to relatively high computational cost.}
\begin{tabular}{cccccc}
\toprule
Hyperparameters & \multicolumn{1}{c}{Accuracy} & \multicolumn{2}{c}{Deformation regularity} & \multicolumn{2}{c}{Consistency}\\
\cmidrule(r){1-1} \cmidrule(r){2-2} \cmidrule(r){3-4} \cmidrule(r){5-6}
$\overline{\lambda}$ & Dice $\uparrow$ & $|J_{\phi}|_{\leq 0} \downarrow$ & $\operatorname{std}(|J_{\phi}|) \downarrow$ &       Cycle $\downarrow$ & Inverse $\downarrow$     \\
\midrule 0.5 & $0.818 (0.031)$ & $\expnumber{2.5}{-2} (\expnumber{6.1}{-3})$ &               $0.48 (0.045)$                & $\expnumber{5.4}{-3} (\expnumber{1.0}{-3})$ & $\expnumber{5.4}{-3} (\expnumber{1.0}{-3})$ \\
    1.0     & $0.815 (0.031)$ & $\expnumber{6.6}{-3} (\expnumber{2.3}{-3})$ &               $0.43 (0.036)$                & $\expnumber{2.8}{-3} (\expnumber{4.9}{-4})$ &$\expnumber{2.8}{-3} (\expnumber{4.9}{-4})$ \\
    5.0     & $0.796 (0.033)$ & $\expnumber{3.9}{-6} (\expnumber{2.1}{-5})$ &               $0.29 (0.021)$                & $\expnumber{4.1}{-4} (\expnumber{5.0}{-5})$ & $\expnumber{4.1}{-4} (\expnumber{5.0}{-5})$ \\
\bottomrule
\end{tabular}
\label{appendix-table:byconstructionicon_hyperparameter_optimization_oasis}
\end{table}

\begin{table}[h]
\centering
\setlength\tabcolsep{3pt}
\scriptsize
\caption{Regularization parameter optimization results for ByConstructionICON calculated on the LPBA40 validation set. The table is interpreted similarly to Table \ref{appendix-table:byconstructionicon_hyperparameter_optimization_oasis}. Value $\lambda = 1.0$ was chosen due to the best overall performance. HD95 metric is not included due to relatively high computational cost.}
\begin{tabular}{ccccccc}
\toprule
Hyperparameters & \multicolumn{2}{c}{Accuracy} & \multicolumn{2}{c}{Deformation regularity} & \multicolumn{2}{c}{Consistency}\\
\cmidrule(r){1-1} \cmidrule(r){2-3} \cmidrule(r){4-5} \cmidrule(r){6-7}
$\overline{\lambda}$ & Dice $\uparrow$ & HD95 $\downarrow$ & $|J_{\phi}|_{\leq 0} \downarrow$ & $\operatorname{std}(|J_{\phi}|) \downarrow$ &       Cycle $\downarrow$ & Inverse $\downarrow$     \\
\midrule 0.5 & $0.686 (0.020)$ &   $6.23 (0.46)$   & $\expnumber{2.8}{-2} (\expnumber{1.2}{-2})$ &               $0.34 (0.042)$                & $\expnumber{3.9}{-3} (\expnumber{1.1}{-3})$ & $\expnumber{3.9}{-3} (\expnumber{1.1}{-3})$ \\
    1.0     & $0.684 (0.018)$ &   $6.30 (0.48)$   & $\expnumber{3.0}{-3} (\expnumber{2.0}{-3})$ &               $0.27 (0.025)$                & $\expnumber{1.2}{-3} (\expnumber{2.5}{-4})$ & $\expnumber{1.2}{-3} (\expnumber{2.5}{-4})$ \\
\bottomrule
\end{tabular}
\label{appendix-table:byconstructionicon_hyperparameter_optimization_lpba40}
\end{table}

\begin{table}[h]
\centering
\setlength\tabcolsep{3pt}
\scriptsize
\caption{Regularization parameter optimization results for ByConstructionICON calculated on the Lung250M-4B validation set. The table is interpreted similarly to Table \ref{appendix-table:byconstructionicon_hyperparameter_optimization_oasis}. Value $\lambda = 5.0$ was chosen due to the best performance.}
\begin{tabular}{ccccccc}
\toprule
Hyperparameters & Accuracy & \multicolumn{2}{c}{Deformation regularity} & \multicolumn{2}{c}{Consistency}\\
\cmidrule(r){1-1}  \cmidrule(r){2-2} \cmidrule(r){3-4} \cmidrule(r){5-6}
$\lambda$ & TRE $\downarrow$ & $|J_{\phi}|_{\leq 0} \downarrow$ & $\operatorname{std}(|J_{\phi}|) \downarrow$ &       Cycle $\downarrow$ & Inverse $\downarrow$     \\
\midrule 0.5 &   $8.76 (4.5)$   &  $\expnumber{2.6}{-3} (\expnumber{3.5}{-3})$  &                $0.32 (0.15)$                & $\expnumber{1.2}{-3} (\expnumber{9.7}{-4})$& $\expnumber{1.2}{-3} (\expnumber{9.7}{-4})$ \\
    1.0     &   $7.63 (4.0)$   &  $\expnumber{3.3}{-4} (\expnumber{9.9}{-4})$  &                $0.25 (0.12)$                & $\expnumber{4.8}{-4} (\expnumber{4.7}{-4})$& $\expnumber{4.8}{-4} (\expnumber{4.7}{-4})$ \\
    3.0     &   $6.38 (3.9)$   & $\expnumber{0.0}{0.0} (\expnumber{0.0}{0.0})$ &                $0.17 (0.10)$                & $\expnumber{1.1}{-4} (\expnumber{8.4}{-5})$ & $\expnumber{1.1}{-4} (\expnumber{8.4}{-5})$ \\
    5.0     &   $6.42 (4.1)$   & $\expnumber{0.0}{0.0} (\expnumber{0.0}{0.0})$ &               $0.15 (0.089)$                & $\expnumber{6.8}{-5} (\expnumber{5.1}{-5})$& $\expnumber{6.8}{-5} (\expnumber{5.1}{-5})$ \\
\bottomrule
\end{tabular}
\label{appendix-table:byconstructionicon_hyperparameter_optimization_lung250m_4b}
\end{table}

\clearpage

\section{Proof of theoretical properties}\label{appendix:theoretical_properties_proof}

While in the main text dependence of the intermediate outputs $d_{1\to1.5}^{(k)}$, $d_{2\to1.5}^{(k)}$, $z^{(k)}_1$, $z^{(k)}_2$, and $\delta^{(k)}$ on the input images $x_A, x_B$ is not explicitly written, throughout this proof we include the dependence in the notation since it is relevant for the proof.

\subsection{Inverse consistent by construction (Theorem \ref{theorem:inverse_consistent})}
\begin{proof}
Inverse consistency by construction follows directly from Equation \ref*{eq:final_deformation}:
\begin{align*}
f_{1\to2}(x_A, x_B) &= d_{1\to1.5}^{(0)}(x_A, x_B) \circ d_{2\to1.5}^{(0)}(x_A, x_B)^{-1}\\
&= \left(d_{2\to1.5}^{(0)}(x_A, x_B) \circ d_{1\to1.5}^{(0)}(x_A, x_B)^{-1}\right)^{-1}\\
&= f_{2\to1}(x_A, x_B)^{-1}
\end{align*}
\end{proof}
Note that due to limited sampling resolution the inverse consistency error is not exactly zero despite of the proof. The same is true for earlier inverse consistent by construction registration methods, as discussed in Section \ref{sec:intro}.

To be more specific, sampling resolution puts a limit on the accuracy of the inverses obtained using deformation inversion layer, and also limits accuracy of compositions if deformations are resampled to their original resolution as part of the composition operation (see Section \ref{sec:inference}). While another possible source could be the fixed point iteration in deformation inversion layer converging imperfectly, that can be proven to be insignificant. As shown by Appendix \ref{appendix:optimal_gamma}, the fixed point iteration is guaranteed to converge, and error caused by the lack of convergence of fixed point iteration can hence be controlled by the stopping criterion. In our experiments we used as a stopping criterion maximum inversion error within all the sampling locations reaching below one hundredth of a voxel, which is very small.

\subsection{Symmetric by construction (Theorem \ref{theorem:symmetric})}
\begin{proof}
We use induction. Assume that for any $x_A$ and $x_B$ at level $k+1$ the following holds: $d^{(k + 1)}_{1\to1.5}(x_A, x_B) = d^{(k + 1)}_{2\to1.5}(x_B, x_A)$. For level $K$ it holds trivially since $d^{(K)}_{1\to1.5}(x_A, x_B)$ and $d^{(K)}_{2\to1.5}(x_A, x_B)$ are defined as identity deformations. Using the induction assumption we have at level $k$:
\begin{equation*}
z_1^{(k)}(x_A, x_B) = h^{(k)}(x_A) \circ d^{(K)}_{1\to1.5}(x_A, x_B) = h^{(k)}(x_A) \circ d^{(K)}_{2\to1.5}(x_B, x_A) = z_2^{(k)}(x_B, x_A)\
\end{equation*}
Then also:
\begin{align*}
\delta^{(k)}(x_A, x_B) &= u^{(k)}(z_1^{(k)}(x_A, x_B), z_2^{(k)}(x_A, x_B)) \circ u^{(k)}(z_2^{(k)}(x_A, x_B), z_1^{(k)}(x_A, x_B))^{-1}
\\
&= u^{(k)}(z_2^{(k)}(x_B, x_A), z_1^{(k)}(x_B, x_A)) \circ u^{(k)}(z_1^{(k)}(x_B, x_A), z_2^{(k)}(x_B, x_A))^{-1}\\
&=  \left[u^{(k)}(z_1^{(k)}(x_B, x_A), z_2^{(k)}(x_B, x_A)) \circ u^{(k)}(z_2^{(k)}(x_B, x_A), z_1^{(k)}(x_B, x_A))^{-1}\right]^{-1}\\
&=  \delta^{(k)}(x_B, x_A)^{-1}\\
\end{align*}

Then we can finalize the induction step:
\begin{align*}
d^{(k)}_{1\to1.5}(x_A, x_B) &= d^{(k + 1)}_{1\to1.5}(x_A, x_B) \circ \delta^{(k)}(x_A, x_B)\\
&= d^{(k + 1)}_{2\to1.5}(x_B, x_A) \circ \delta^{(k)}(x_B, x_A)^{-1} = d^{(k)}_{2\to1.5}(x_B, x_A)
\end{align*}

From this follows that the method is symmetric by construction:
\begin{align*}
f_{1\to2}(x_A, x_B) &= d_{1\to1.5}^{(0)}(x_A, x_B) \circ d_{2\to1.5}^{(0)}(x_A, x_B)^{-1}\\
&= d_{2\to1.5}^{(0)}(x_B, x_A) \circ d_{1\to1.5}^{(0)}(x_B, x_A)^{-1} = f_{2\to1}(x_B, x_A)
\end{align*}
\end{proof}

The proven relation holds exactly.

\subsection{Topology preserving (Theorem \ref{theorem:topology_preserving})}

\begin{proof}
    As shown by Appendix \ref{appendix:optimal_gamma}, each $u^{(k)}$ produces topology preserving (everywhere positive Jacobian determinants) deformations (architecture of $u^{(k)}$ is described in Section \ref{sec:topology_preserving_u}). Since the overall deformation is composition of multiple outputs of $u^{(k)}$ and their inverses, the whole deformation has also everywhere positive Jacobians since the Jacobian determinants at each point can be obtained as a product of the Jacobian determinants of the composed deformations.
\end{proof} 

The inveritibility is not perfect if the compositions of $u^{(k)}$ and their inverses are resampled to the input image resolution, as is common practice in image registration. However, invertibility everywhere can be achieved by storing all the individual deformations and evaluating the composed deformation as their true composition (see Section \ref{sec:inference} on inference variants and the results in Section \ref{sec:results}).

\section{Deriving the optimal bound for control points}\label{appendix:optimal_gamma}

\subsection{Proof summary}

As discussed in Section \ref{sec:topology_preserving_u}, we limit absolute values of the predicted cubic spline control points defining the displacement field by a hard constraint $\gamma^{(k)}$ for each resolution level $k  \in \{0,\dots,K - 1\}$. We want to find optimal $\gamma^{(k)}$ which ensures invertibility of individual deformations and convergence of the fixed point iteration in deformation inversion layer. Note that the proof provides a significantly shorter and more general proof of the theorems 1 and 4 in \citep{choi2000injectivity}.

We start by showing the optimal bound in continuous case (infinite resolution), and then extend it to our discrete case where values between the B-spline samples are defined by linear interpolation.

The derived continuous bound $\gamma$ equals the reciprocal of the maximum possible matrix $\infty$ norm of the local Jacobian matrices over all possible generated displacement fields. The matrix norm of a local Jacobian matrix equals the local Lipschitz constant and hence the bound ensures that the Lipschitz constant with respect to the $\infty$ norm stays under one, which by \citep{chen2008simple} guarantees both invertibility and convergence of the fixed point iteration. However, a straightforward formula for the bound is intractable computationally in three dimensions, and we simplify it by showing that due to symmetries the absolute value around the derivatives can be removed from the function being maximixed, yielding a formula which can be evaluated exactly. We further show with a counter example that the proposed bound is tight.

The bound equals the bound obtained by \citep{choi2000injectivity} using a very different approach, further validating the result (their proof does not cover the convergence of the fixed point iteration and is less general).

Note that the bound ensures positivity of the Jacobians, not only invertibility, since we are limiting displacements (and zero displacement deformation has trivially positive Jacobian).

\subsection{Proof of the continuous case}

By \citep{chen2008simple} a deformation field is invertible by the proposed deformation inversion layer (and hence invertible in general) if its displacement field is contractive mapping with respect to some norm (the convergence is then also with respect to that norm). In finite dimensions convergence in any p-norm is equal and hence we should choose the norm which gives the loosest bound.

Let us choose $||\cdot||_{\infty}$ norm for our analysis, which, as it turns out, gives the loosest possible bound. Then the Lipschitz constant of a displacement field is equivalent to the maximum $||\cdot||_{\infty}$ operator norm of the local Jacobian matrices of the displacement field.

Since for matrices $||\cdot||_{\infty}$ norm corresponds to maximum absolute row sum it is enough to consider one component of the displacement field.

Let $B: \mathbb{R}\to\mathbb{R}$ be a centered cardinal B-spline of some degree (actually any continuous almost everywhere differentiable function with finite support is fine) and let us consider an infinite grid of control points $\phi: \mathbb{Z}^n \to \mathbb{R}$ where $n$ is the dimensionality of the displacement field. For notational convinience, let us define a set $N:=\{1, \dots, n\}$.

Now let $f_\phi$ be the $n$-dimensional displacement field (or one component of it) defined by the control point grid:

\begin{equation}
    f_\phi(x) = \sum_{\alpha\in\mathbb{Z}^n} \phi(\alpha) \prod_{i\in N}B(x_i - \alpha_i)
\end{equation}

Note that since the function $B$ has finite support the first sum over $\mathbb{Z}^n$ can be defined as a finite sum for any $x$ and is hence well-defined. Also, without loss of generality it is enough to look at region $x \in [0, 1]^n$ due to the unit spacing of the control point grid.

For partial derivatives $\frac{\partial f_\phi}{\partial x_j}: \mathbb{R}^n \to \mathbb{R}$ we have
\begin{equation}
    \frac{\partial f_\phi}{\partial x_j}(x) := \sum_{\alpha\in \mathbb{Z}^n} \phi({\alpha})\ B'(x_j - \alpha_j) \prod_{i\in N\setminus \{j\}} B(x_i - \alpha_i) = \sum_{\alpha\in\mathbb{Z}^n} \phi_{\alpha} D^j(x - \alpha)
\end{equation}
where $D^j(x - \alpha) := B'(x_j - \alpha_j) \prod_{i\in N\setminus \{j\}} B(x_i - \alpha_i)$.

Following the power set notation, let us denote control points limited to some set $S \subset \mathbb{R}$ as $S^{\mathbb{Z}^n}$. That is, if $\phi \in S^{\mathbb{Z}^n}$, then for all $\alpha \in \mathbb{Z}^n$, $\phi(\alpha) \in S$.

\begin{lemma}\label{appendix-lemma:contraction_tilde}
For all $\phi \in\ ]-1/\tilde{K}_n, 1/\tilde{K}_n[^{\mathbb{Z}^n}$, $f_\phi$ is a contractive mapping with respect to the $||\cdot||_{\infty}$ norm, where
\begin{equation}
    \tilde{K}_n := \max_{\begin{subarray}{l}x \in [0, 1]^n\\\tilde{\phi} \in [-1, 1]^{\mathbb{Z}^n}\end{subarray}}\ \sum_{j\in N} \left| \frac{\partial f_{\tilde{\phi}}}{\partial x_j}(x) \right|.
\end{equation}
\end{lemma}
\begin{proof}
For all $x\in[0, 1]^n$,  $\phi \in\ ]-1/\tilde{K}_n, 1/\tilde{K}_n[^{\mathbb{Z}^n}$
\begin{equation}
\begin{aligned}
    \sum_{j\in N} \left| \frac{\partial f_\phi}{\partial x_j}(x) \right| &<
    \max_{\begin{subarray}{l}\tilde{x} \in [0, 1]^n\\\tilde{\phi} \in [-1/\tilde{K}_n, 1/\tilde{K}_n]^{\mathbb{Z}^n}\end{subarray}}\ \sum_{j\in N} \left| \frac{\partial f_{\tilde{\phi}}}{\partial \tilde{x}_j}(\tilde{x}) \right|\\
    &= \max_{\begin{subarray}{l}\tilde{x} \in [0, 1]^n\\\tilde{\phi} \in [-1, 1]^{\mathbb{Z}^n}\end{subarray}}\ \sum_{j\in N} \left| \frac{\partial f_{\tilde{\phi}/\tilde{K}_n}}{\partial \tilde{x}_j}(\tilde{x}) \right|\\
    &= \frac{1}{\tilde{K}_n}\ \max_{\begin{subarray}{l}\tilde{x} \in [0, 1]^n\\\tilde{\phi} \in [-1, 1]^{\mathbb{Z}^n}\end{subarray}}\ \sum_{j\in N} \left| \frac{\partial f_{\tilde{\phi}/\tilde{K}_n}}{\partial \tilde{x}_j}(\tilde{x}) \right| = \frac{\tilde{K}_n}{\tilde{K}_n} = 1.
\end{aligned}
\end{equation}
Sums of absolute values of partial derivatives are exactly the $||\cdot||_{\infty}$ operator norms of the local Jacobian matrices of $f$, hence $f$ is a contraction.
\end{proof}

\begin{lemma}\label{appendix-lemma:no_abs_needed}
    For any $k \in N$, $x \in [0, 1]^n,\ \phi \in [-1, 1]^{\mathbb{Z}^n}$, we can find some $\tilde{x} \in [0, 1]^n, \tilde{\phi}\in [-1, 1]^{\mathbb{Z}^n}$ such that
    \begin{equation}
        \frac{\partial f_{\tilde{\phi}}}{\partial x_j}(\tilde{x}) =
        \begin{cases}
            -\frac{\partial f_\phi}{\partial x_j}(x)&\quad \text{for}\ j = k\\
            \frac{\partial f_\phi}{\partial x_j}(x)&\quad \text{for}\ j \in N\setminus \{k\}.
        \end{cases}
    \end{equation}
\end{lemma}
\begin{proof}
The B-splines are symmetric around origin:
\begin{equation}
    B(x) = B(-x) \implies B'(x) = -B'(-x)
\end{equation}

Let us propose
\begin{equation}
    \tilde{x}_i :=
    \begin{cases}
        1 - x_i, &\quad \text{when}\ i \in N \setminus k\\
        x_i, &\quad \text{when}\ i = k
    \end{cases}
\end{equation}

and $\tilde{\phi}: \mathbb{Z}^n \to \mathbb{R}$ as $\tilde{\phi}(\alpha) := -\phi(g(\alpha))$ where $g: \mathbb{Z}^n \to \mathbb{Z}^n$ is a bijection defined as follows:
\begin{equation}
    g(\alpha)_i :=
    \begin{cases}
        1 - \alpha_i, &\quad \text{when}\ i \in N \setminus k\\
        \alpha_i, &\quad \text{when}\ i = k.
    \end{cases}
\end{equation}

Then for all $\alpha \in \mathbb{Z}^n$:
\begin{equation}
    \begin{aligned}
        D^k(\tilde{x} - \alpha) &= B'(\tilde{x}_k - \alpha_k) \prod_{i\in N\setminus \{k\}} B(\tilde{x}_i - \alpha_i)\\
        &= B'(x_k - g(\alpha)_k) \prod_{i\in N\setminus \{k\}} B(-(x_i - g(\alpha)_i))\\
        &= D^k(x - g(\alpha))
    \end{aligned}
\end{equation}

which gives

\begin{equation}
    \begin{aligned}
        \frac{\partial f_{\tilde{\phi}}}{\partial \tilde{x}_k}(\tilde{x}) &= \sum_{\alpha\in\mathbb{Z}^n} \tilde{\phi}(\alpha) D^k(\tilde{x} - \alpha)\\
        &= \sum_{\alpha\in\mathbb{Z}^n} -\phi(g(\alpha)) D^k(x - g(\alpha))&&\quad \text{g is bijective}\\
        &= -\frac{\partial f_\phi}{\partial x_k}(x).
    \end{aligned}
\end{equation}

And for all $j \in N \setminus \{k\}$, $\alpha \in \mathbb{Z}^n$

\begin{equation}
    \begin{aligned}
        D^j(\tilde{x} - \alpha) &= B'(\tilde{x}_j - \alpha_j) \prod_{i\in N\setminus \{j\}} B(\tilde{x}_i - \alpha_i)\\
        &= B'(\tilde{x}_j - \alpha_j)\ B(\tilde{x}_k - \alpha_k) \prod_{i\in N\setminus \{j, k\}} B(\tilde{x}_i - \alpha_i)\\
        &= B'(-(x_j - g(\alpha)_j))\ B(x_k - g(\alpha)_k) \prod_{i\in N\setminus \{j, k\}} B(-(x_i - g(\alpha)_i))\\
        &= -B'(x_j - g(\alpha)_j)\ B(x_k - g(\alpha)_k) \prod_{i\in N\setminus \{j, k\}} B(x_i - g(\alpha)_i)\\
        &= -B'(x_j - g(\alpha)_j) \prod_{i\in N\setminus \{j\}} B(x_i - g(\alpha)_i)\\
        &= - D^k(x - g(\alpha))
    \end{aligned}
\end{equation}

which gives for all $j \in N \setminus \{k\}$
\begin{equation}
    \begin{aligned}
        \frac{\partial f_{\tilde{\phi}}}{\partial \tilde{x}_j}(\tilde{x}) &= \sum_{\alpha\in\mathbb{Z}^n} \tilde{\phi}(\alpha) D^j(\tilde{x} - \alpha)\\
        &= \sum_{\alpha\in\mathbb{Z}^n} -\phi(g(\alpha)) - D^j(x - g(\alpha))&&\quad \text{g is bijective}\\
        &= \frac{\partial f_\phi}{\partial x_j}(x).
    \end{aligned}
\end{equation}
\end{proof}

\begin{theorem}\label{appendix-theorem:contraction}
For all $\phi \in\ ]-1/K_n, 1/K_n[^{\mathbb{Z}^n}$, $f_\phi$ is a contractive mapping with respect to the $||\cdot||_{\infty}$ norm, where
\begin{equation}
    K_n := \max_{\begin{subarray}{l}x \in [0, 1]^n\end{subarray}}\ \sum_{\alpha\in \mathbb{Z}^n} \left|\sum_{j\in N} D^j_{\alpha}(x)\right|.
\end{equation}
\end{theorem}
\begin{proof}
Let us show that $K_n = \tilde{K}_n$.
\begin{equation}
    \begin{aligned}
        \tilde{K}_n &= \max_{\begin{subarray}{l}x \in [0, 1]^n\\\phi \in [-1, 1]^{\mathbb{Z}^n}\end{subarray}}\ \sum_{j\in N} \left| \frac{\partial f_\phi}{\partial x_j}(x) \right|&&\\
        &= \max_{\begin{subarray}{l}x \in [0, 1]^n\\\phi \in [-1, 1]^{\mathbb{Z}^n}\end{subarray}}\ \sum_{j\in N} \frac{\partial f_\phi}{\partial x_j}(x)&&\quad \text{(Lemma \ref{appendix-lemma:no_abs_needed})}\\
        &= \max_{\begin{subarray}{l}x \in [0, 1]^n\\\phi \in [-1, 1]^{\mathbb{Z}^n}\end{subarray}}\ \sum_{j\in N} \sum_{\alpha\in\mathbb{Z}^n} \phi_{\alpha}\ D^j(x - \alpha) &&\\
        &= \max_{\begin{subarray}{l}x \in [0, 1]^n\\\phi \in [-1, 1]^{\mathbb{Z}^n}\end{subarray}}\ \sum_{\alpha\in\mathbb{Z}^n} \phi_{\alpha}\ \sum_{j\in N} D^j(x - \alpha)&&\\
        &= \max_{\begin{subarray}{l}x \in [0, 1]^n\end{subarray}}\ \sum_{\alpha\in\mathbb{Z}^n} \left|\sum_{j\in N} D^j_{\alpha}(x)\right| = K_n&&\\
    \end{aligned}
\end{equation}
The last step follows from the obvious fact that the sum is maximized when choosing each $\phi_\alpha$ to be either $1$ or $-1$ based on the sign of the inner sum $\sum_{j\in N} D^j(x - \alpha)$.

By Lemma \ref{appendix-lemma:contraction_tilde} $f$ is then a contractive mapping with respect to the $||\cdot||_{\infty}$ norm.
\end{proof}

Theorem \ref{appendix-theorem:contraction} proves that if we limit the control point absolute values to be less than $1 / K_n$, then the resulting deformation is invertible by the fixed point iteration. Also, approximating $K_n$ accurately is possible at least for $n \leq 3$. Subset of $\mathbb{Z}^n$ over which the sum needs to be taken depends on the support of the function $B$ which again depends on the degree of the B-splines used.

Next we want to show that the obtained bound is also tight bound for invertibility of the deformation. That also then shows that $||\cdot||_{\infty}$ norm gives the loosest possible bound.

Since $f_\phi$ corresponds only to one component of a displacement field, let us consider a fully defined displacement field formed by stacking $n$ number of $f_\phi$ together. Let us define
\begin{equation}
    g_\phi(x) := (f_\phi)_{i\in N}.
\end{equation}
\begin{theorem}
    There exists $\phi \in [-1/K_n, 1/K_n]^{\mathbb{Z}^n}$, $x \in [0, 1]^n$ s.t. $\det \left(\frac{\partial g_\phi}{\partial x} + I\right)(x) = 0$ where $\frac{\partial g_\phi}{\partial x}$ is the Jacobian matrix of $g_\phi$ and $I$ is the identity matrix. 
\end{theorem}
\begin{proof}
By Lemma \ref{appendix-lemma:no_abs_needed} and Theorem \ref{appendix-theorem:contraction} there exists $x \in [0, 1]^n$ and $\tilde{\phi} \in [-1, 1]^{\mathbb{Z}^n}$ such that
\begin{equation}
    \sum_{j\in N} \frac{\partial f_{\tilde{\phi}}}{\partial x_j}(x) = -K_n
\end{equation}
where all $\frac{\partial f_{\tilde{\phi}}}{\partial x_j}(x) < 0$.

Let us define $\phi := \tilde{\phi} / K_n \in [-1/K_n, 1/K_n]^{\mathbb{Z}^n}$. Then
\begin{equation}
    \sum_{j\in N} \frac{\partial f_{\phi}}{\partial x_j}(x) = -1.
\end{equation}

Now let $y \in \mathbb{R}^n$ be a vector consisting only of values $1$, that is $y=:(1)_{i\in N}$. Then one has
\begin{equation}
    \begin{aligned}
        \left(\frac{\partial g_{{\phi}}}{\partial x} + I\right)(x) y &= \left(\frac{\partial g_{{\phi}}}{\partial x}(x)\right)y + y\\
        &= \left(\sum_{j\in N}1\ \frac{\partial f_{{\phi}}}{\partial x_j}(x)\right)_{i\in N} + (1)_{i\in N}\\
        &= (-1)_{i\in N} + (1)_{i\in N} = 0.
    \end{aligned}
\end{equation}
In other words $y$ is an eigenvector of $\left(\frac{\partial g_{{\phi}}}{\partial x} + I\right)(x)$ with eigenvalue $0$ meaning that the determinant of $\left(\frac{\partial g_{{\phi}}}{\partial x} + I\right)(x)$ is also $0$.
\end{proof}

The proposed bound is hence the loosest possible since the deformation can have zero Jacobian at the bound, meaning it is not invertible.

\subsection{Sampling based case (Equation \ref{eq:optimal_gamma})}

The bound used in practice, given in Equation \ref{eq:optimal_gamma}, is slightly different to the bound proven in Theorem \ref{appendix-theorem:contraction}. The reason is that for computational efficiency we do not use directly the cubic B-spline representation for the displacement field but instead take only samples of the displacement field in the full image resolution (see Appendix \ref{sec:topology_preserving_u}), and use efficient bi- or trilinear interpolation for defining the intermediate values. As a result the continuous case bound does not apply anymore.

However, finding the exact bounds for our approximation equals evaluating the maximum in Theorem \ref{appendix-theorem:contraction} over a finite set of sampling locations and replacing $D^j(x)$ with finite difference derivatives. The mathematical argument for that goes almost identically and will not be repeated here. However, to justify using finite difference derivatives, we need the following two trivial remarks:
\begin{itemize}
    \item When defining a displacement field using grid of values and bi- or trilinear interpolation, the highest value for $||\cdot||_\infty$ operator norm is obtained at the corners of each interpolation patch.
    \item Due to symmetry, it is enough to check derivative only at one of $2^n$ corners of each bi- or trilinear interpolation patch in computing the maximum (corresponding to finite difference derivative in only one direction over each dimension).
\end{itemize}

Maximum is evaluated over the relative sampling locations with respect to the resolution of the control point grid (which is in the resolution of the features $z_1^{(k)}$ and $z_2^{(k)}$). The exact sampling grid depends on how the sampling is implemented (which is an implementation detail), and in our case we used the locations $X := \{1 / 2 + \frac{1}{2^{k + 1}} + \frac{i}{2^k}\ |\ i \in \mathbb{Z}\}^n \cap [0, 1]^n$ which have, without loss of generality, been again limited to the unit cube.

No additional insights are required to show that the equation \ref{eq:optimal_gamma} gives the optimal bound. For concrete values of $K^{(k)}_n$, see Table \ref{appendix-table:K_values}.

\begin{table}
\centering
\setlength\tabcolsep{3pt}
\scriptsize
\caption{Values of $K_2$ and $K_3$ for different sampling rates with respect to the control point grid. The bound for $\text{Sampling rate} = \infty$ is from \citep{choi2000injectivity}. For each resolution level we define the maximum control point absolute values $\gamma^{(k)}$ as $0.99 \times \frac{1}{K^{(k)}_n}$ (in our experiments we have $n=3$ dimensional data). Codebase contains implementation for computing the value for other $k$.}
\begin{tabular}{llcc}
\toprule
k & Sampling rate                    &    $K^{(k)}_2$    &    $K^{(k)}_3$    \\
\midrule
0 & 1                                    & 2.222222222 & 2.777777778 \\
1 & 2                                    & 2.031168620 & 2.594390728 \\
2 & 4                                    & 2.084187826 & 2.512366240 \\
3 & 8                                    & 2.063570023 & 2.495476474 \\
4 & 16                                   & 2.057074951 & 2.489089713 \\
5 & 32                                   & 2.052177394 & 2.484247818 \\
6 & 64                                   & 2.049330491 & 2.481890143 \\
7 & 128                                  & 2.047871477 & 2.480726430 \\
8 & 256                                  & 2.047136380 & 2.480102049 \\
$\infty$  & $\infty$ & 2.046392675 & 2.479472335 \\
\bottomrule
\end{tabular}
\label{appendix-table:K_values}
\end{table}
\clearpage
\section{Deformation inversion layer memory usage}\label{appendix:memory}

We conducted an experiment on the memory usage of the deformation inversion layer compared to the stationary velocity field (SVF) framework \citep{arsigny2006log} since SVF framework could also be used to implement the suggested architecture in practice.

With the SVF framework one could slightly simplify the deformation update Equation \ref{eq:update_deformation_formula} to the form
\begin{equation}\label{appendix-eq:svf_symmetric_update}
    U^{(k)} := \exp(u^{(k)}(z_1^{(k)}, z_2^{(k)}) - u^{(k)}(z_2^{(k)}, z_1^{(k)}))
\end{equation}
where $\exp$ is the SVF integration (corresponding to Lie algebra exponentiation), and $u^{(k)}$ now predicts an auxiliary velocity field. We compared memory usage of this to our implementation, and used the implementation by \citet{dalca2018unsupervised} for SVF integration.

The results are shown in Table \ref{appendix-table:svf_deformation_inversion_layer_memory_comparison}. Our version implemented using the deformation inversion layer requires $5$ times less data to be stored in memory for the backward pass compared to the SVF integration. The peak memory usage during the inversion is also slightly lower. The memory saving is due to the memory efficient back-propagation through the fixed point iteration layers, which requires only the final inverted volume to be stored for backward pass. Since our architecture requires two such operations for each resolution level ($U^{(k)}$ and its inverse), the memory saved during training is significant.

\begin{table}[h]
\centering
\setlength\tabcolsep{3pt}
\scriptsize
\caption{\textbf{Memory usage comparison between deformation inversion layer and stationary velocity field (SVF) based implementations.} The comparison is between executing Equation \ref{eq:update_deformation_formula} using deformation inversion layers and executing Equation \ref{appendix-eq:svf_symmetric_update} using SVF integration implementation by \citet{dalca2018unsupervised}. Between passes memory usage refers to the amount memory needed for storing values between forward and backward passes, and peak memory usage refers to the peak amount of memory needed during forward and backward passes. A volume of shape $(256, 256, 256)$ with $32$ bit precision was used. We used $7$ scalings and squarings for the SVF integration.}
\begin{tabular}{lcc}
\toprule
Method & Between passes memory usage (GB) $\downarrow$ & Peak memory usage (GB) $\downarrow$\\
\midrule
Deformation inversion layer & $\bm{0.5625}$ & $\bm{3.9375}$\\
SVF integration & $2.8125$ & $4.125$\\
\bottomrule
\end{tabular}
\label{appendix-table:svf_deformation_inversion_layer_memory_comparison}
\end{table}

\section{Extended related work}
\label{appendix:related_work}
In this appendix we aim to provide a more thorough analysis of the related work, by introducing in more detail the works that we find closest to our method, and explaining how our method differs from those.

\subsection{Classical registration methods}

Classical registration methods, as opposed to the learning based methods, optimize the deformation independently for any given single image pair. For this reason they are sometimes called optimization based methods.

\textbf{DARTEL} by \citet{ashburner2007fast} is a classical optimization based registration method built on top of the stationary velocity field (SVF) \citep{arsigny2006log} framework offering symmetric by construction, inverse consistent, and topology preserving registration. The paper is to our knowledge the first symmetric by construction, inverse consistent, and topology preserving registration method.

\textbf{SyN} by \citet{avants2008symmetric} is another classical symmetric by construction, inverse consistent, and topology preserving registration method. The properties are achieved by the Large Deformation Diffeomorphic Metric Mapping LDDMM framework \citep{beg2005computing} in which diffeormphisms are generated from time-varying velocity fields (as opposed to the stationary ones in the SVF framework). The LDDMM framework has not been used much in unsupervised deep learning for generating diffeomorphims due to its computational cost, but some works do exist \citep{shen2019region, ramon2022lddmm, wang2023metamorph}, and others which are inspired by the framework but make significant modifications, and as a result lose the by construction topology preserving properties \citep{wang2020deepflash, wu2022nodeo, joshi2023r2net}. SyN is to our knowledge the first work suggesting matching the images in the intermediate coordinates for achieving symmetry, the idea which was also used in our work. However, the usual implementation of SyN in ANTs \citep{avants2009advanced} is not as a whole symmetric since the affine registration is not applied in a symmetric manner.

SyN has performed well in evaluation studies between different classical registration methods. e.g. \citep{klein2009evaluation}. However, it is significantly slower than the strong baselines included in our study, and has already earlier been compared with those\citep{balakrishnan2019voxelmorph, mok2020fast, mok2020large}, and hence was not included in our study.

\subsection{Deep learning methods (earlier work)}

Unlike the optimization based methods above, deep learning methods train a neural network that, for two given input images, outputs a deformation directly. The benefits of this class of methods include the significant speed improvement and more robust performance (avoiding local optima) \citep{de2019deep}. Our model belongs to this class of methods.

\textbf{SYMNet} by \citet{mok2020fast} uses as single forward pass of a U-Net style neural network to predict two stationary velocity fields, $v_{1\to1.5}$ and $v_{2\to1.5}$ (in practice two 3 channeled outputs are extracted from the last layer features using separate convolutions). The stationary velocity fields are integrated into two half-way deformations (and their inverses). The training loss matches the images both in the intermediate coordinates and in the original coordinate spaces (using the composed full deformations). While the network is a priori symmetric with respect to the input images, changing the input order of of the images (concatenating the inputs in the opposite order for the U-Net) can in principle result in any two $v_{1\to1.5}$ and $v_{2\to1.5}$ (instead of swapping them), meaning that the method is not symmetric by construction as defined in Section \ref{sec:intro} (this is confirmed by the cycle consistency experiments).

Our method does not use the stationary velocity field (SVF) framework to invert the deformations, but instead uses the novel deformation inversion layer. Also, SYMNet does not employ the multi-resolution strategy. The use of intermediate coordinates is similar to our work.

\textbf{MICS} by \citet{estienne2021mics} uses a shared convolutional encoder $E$ to encode both input images into some feature representations $E(x_A)$ and $E(x_B)$. A convolutional decoder network $D$ is then used to extract gradient volumes (constrained to contain only positive values) of the deformations for both forward and inverse deformations with formulas

\begin{equation}\label{appendix-eq:MICS}
    \nabla f(x_A, x_B)_{1\to2} = D(E(x_A) - E(x_B))\text{ and }\nabla f(x_A, x_B)_{2\to1} = D(E(x_B) - E(x_A)).
\end{equation}

The final deformations are obtained from the gradient volumes by a cumulative sum operation. Since gradients are constrained to be positive the resulting deformation will be roughly invertible. However, as stated in their work, this only puts a positive constraint on the diagonal of the Jacobians, not on its determinant, unlike our work which guarantees positive determinants (Theorem \ref{theorem:topology_preserving}).

While MICS is symmetric by construction in the sense that swapping $x_A$ and $x_B$ will result in swapping the predicted forward and inverse deformations, this symmetry is achieved by subtraction (Equation \ref{appendix-eq:MICS}) instead of mathematical inverse operation (as in our work, Equation \ref{eq:anti-symmetric_formulation}). As a result the predicted "forward" and "inverse" deformations are not actually by construction constrained to be forward and inverse deformations of each other. MICS uses a loss to enforce this. Also, while MICS employs a multi-step approach, the symmetric by construction property is lost over the whole architecture due to not using the intermediate coordinate approach employed by our work.

\textbf{Additional baselines}: In addition to SYMNet, \textbf{cLapIRN}\citep{mok2021conditional} was chosen as a baseline because it was the best method on OASIS dataset in the Learn2Reg challenge \citep{hering2022learn2reg}. It employs a standard and straightforward multi-resolution approach, and is not topology preserving, inverse consistent, or symmetric. Apart from the multi-resolution approach, it is not methodologically close to our method. \textbf{VoxelMorph} \citep{balakrishnan2019voxelmorph} is a standard baseline in deep learning based unsupervised registration, and it is based on a straightforward application of U-Net architecture to image registration. 

\subsection{Deep learning methods (recent methods parallel with our work)}

These methods are very recent deep learning based registrations methods which have been developed independently of and in parallel with our work. A comparison with these methods and our model is therefore a fruitful topic for future research.

\textbf{\citet{iglesias2023ready}} uses a similar approach to achieve symmetry as our work but in the stationary velocity field (SVF) framework. In SVF framework, given some velocity field $v$, we have the property that $\exp(v) = \exp(-v)^{-1}$ where $\exp$ represents the integration operation \citep{arsigny2006log} generating the deformation from the velocity field. Hence to obtain symmetric by construction method, one can modify the formula \ref{eq:anti-symmetric_formulation} to the form
\begin{equation}
    f(x_A, x_B) := \exp(u(x_A, x_B) - u(x_B, x_A)).
\end{equation}
which will result in a method which is symmetric by construction, inverse consistent, and topology preserving. We measure memory usage against this formulation in Appendix \ref{appendix:memory} and show that our formulation using the novel deformation inversion layer requires storing $5$ times less memory for the backward pass. Their method includes only a single registration step, and not the robust multi-resolution architecture like ours.

\textbf{\citet{greer2023inverse}} extends the approach the approach by \citet{iglesias2023ready} to multi-step formulation in very similar way to how we construct our multi-resolution architecture by using the intermediate coordinates (square roots of deformations). They also employ the SVF framework, as opposed to our work which uses the deformation inversion layers, which, as shown in Appendix \ref{appendix:memory}, requires storing $5$ times less memory for the backward pass, which is significant for being able to train the multi-step network with many steps on large images (such as our OASIS raw data). Also, their paper treats the separate steps of multi-step architecture as independent whereas we develop very efficient multi-resolution formulation based on first extracting the multi-resolution features using ResNet-style encoder.

\textbf{\citet{zhang2023symmetric}} propose a symmetric by construction multi-resolution architecture that is similar to ours. Differences include again using the SVF framework, as well as applying losses at different resolutions, as opposed to us only applying them at the final level.

\clearpage

\section{Deformation inversion layer practical convergence}\label{appendix:deformation_inversion_iteration_counts}

We conducted an experiment on the fixed point iteration convergence in the deformation inversion layers with the model trained on OASIS dataset. The results can be seen in Figure \ref{appendix-fig:deformation_inversion_iteration_counts}. The main result was that in the whole OASIS test set of 9591 pairs not a single deformation required more than 8 iterations for convergence. Deformations requiring 8 iterations were only $0.05\%$ of all the deformations and a significant majority of the deformations ($96\%$) required $2$ to $5$ iterations. In all the experiments, including this one, the stopping criterion for the iterations was maximum displacement error within the whole volume reaching below one hundredth of a voxel, which is a very small error.

\begin{figure}[h]
\centering
\includegraphics[width=1.0\textwidth]{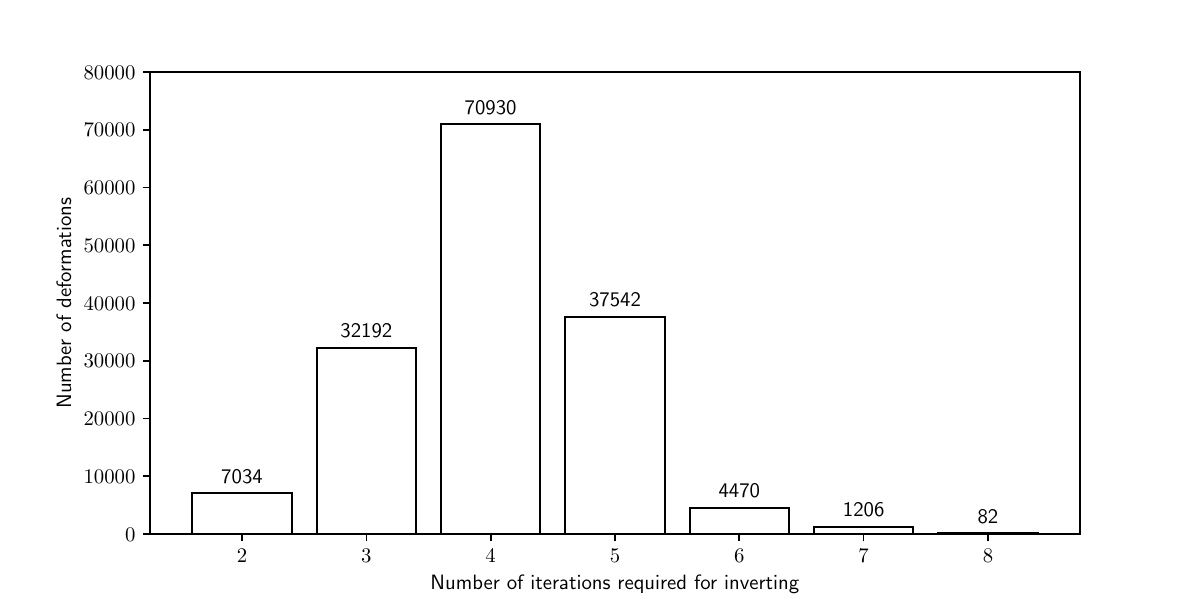}
\caption{\textbf{Number of fixed point iterations required for convergence in deformation inversion layers with the model trained on OASIS dataset.} The stopping criterion for the fixed point iteration was maximum displacement error within the whole volume reaching below one hundredth of a voxel. All deformation inversions for the whole OASIS test set are included.}
\label{appendix-fig:deformation_inversion_iteration_counts}
\end{figure}
\clearpage
\section{Additional visualizations}\label{appendix:result_visualization}

Figures \ref{appendix-fig:inverse_consistency_visualization}, and \ref{appendix-fig:cycle_consistency_visualization}
visualize the differences in inverse consistency, and cycle consistency respectively.

Figures \ref{appendix-fig:dice_scores_oasis} and \ref{appendix-fig:dice_scores_lpba40} visualize dice scores for individual anatomical regions for both OASIS and LPBA40 datasets. VoxelMorph and SYMNet perform systematically worse than our method, while cLapIRN and our method perform very similarly on most regions.

Figure \ref{appendix-fig:detailed_deformation_example} visualizes how the deformation is being gradually updated during the multi-resolution architecture.

\begin{figure}[h]
\centering
\includegraphics[width=0.9\textwidth]{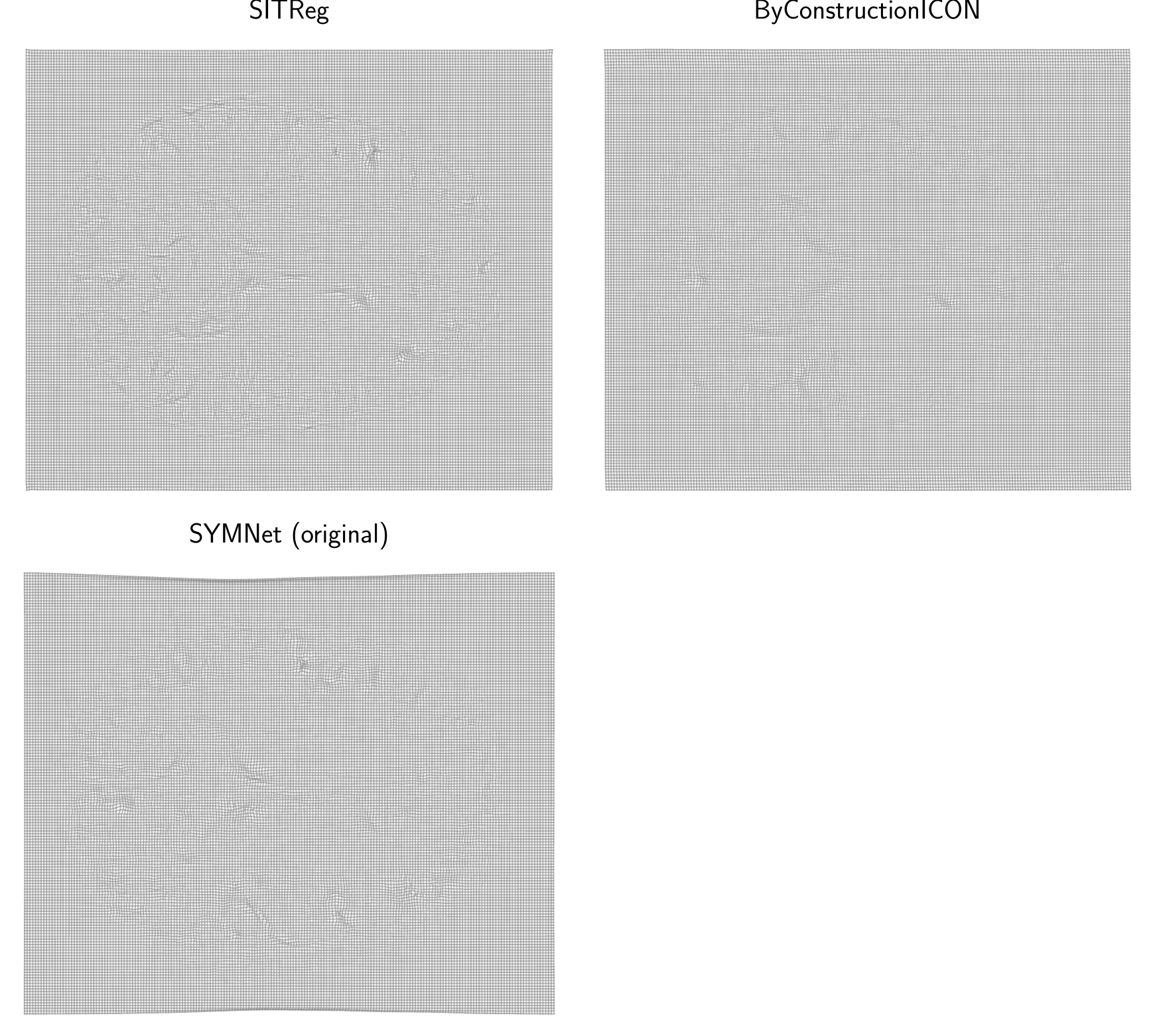}
\caption{\textbf{Visual inverse consistency comparison.} The deformation $\bm{f(x_A, x_B)_{1\to2} \circ f(x_A, x_B)_{2\to1}}$ is visualized for SITReg and SYMNet models for a single image pair in LPBA40 experiment. Since cLapIRN and VoxelMorph do not generate explicit inverses, they are not included in the figure. Ideally, $f(x_A, x_B)_{1\to2} \circ f(x_A, x_B)_{2\to1}$ should equal the identity mapping, and as can be seen, the property is well fulfilled by all of the three methods. Only one axial slice of the predicted 3D deformation is visible.}
\label{appendix-fig:inverse_consistency_visualization}
\end{figure}

\begin{figure}[h]
\centering
\includegraphics[width=0.8\textwidth]{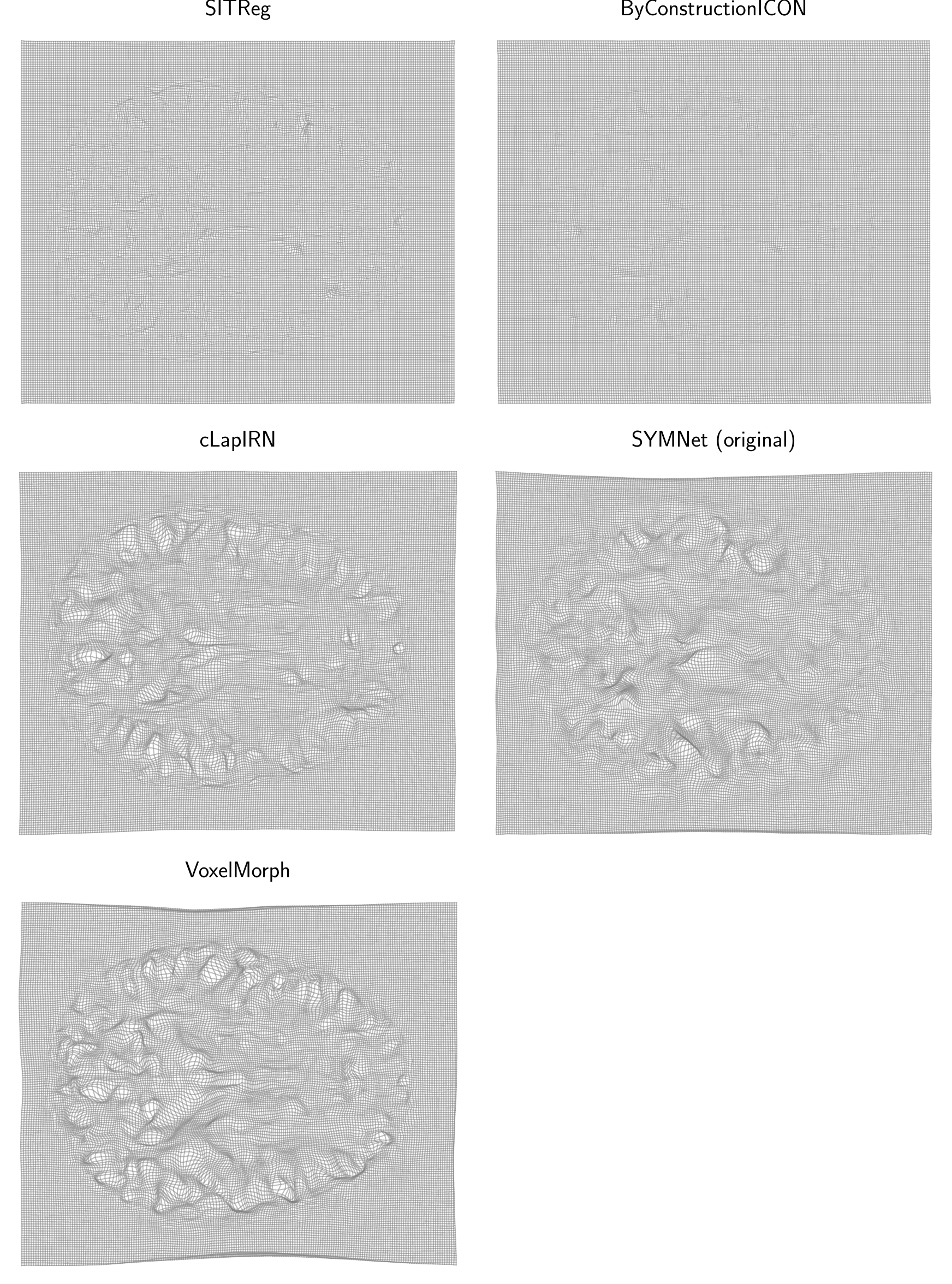}
\caption{\textbf{Visual cycle consistency comparison.} The deformation composition $f(x_A, x_B) \circ f(x_B, x_A)$ is visualized for each model for a single image pair in LPBA40 experiment. Ideally, changing the order of the input images should result in the same coordinate mapping but in the inverse direction, since anatomical correspondence is not dependent on the input order. In other words, the deformation composition $f(x_A, x_B) \circ f(x_B, x_A)$ should equal the identity deformation. As can be seen, the property is only fulfilled (up to small sampling errors) by our method and ByConstructionICON. Only one axial slice of the predicted 3D deformation is shown.}
\label{appendix-fig:cycle_consistency_visualization}
\end{figure}

\begin{figure}[h]
\centering
\includegraphics[width=1.0\textwidth]{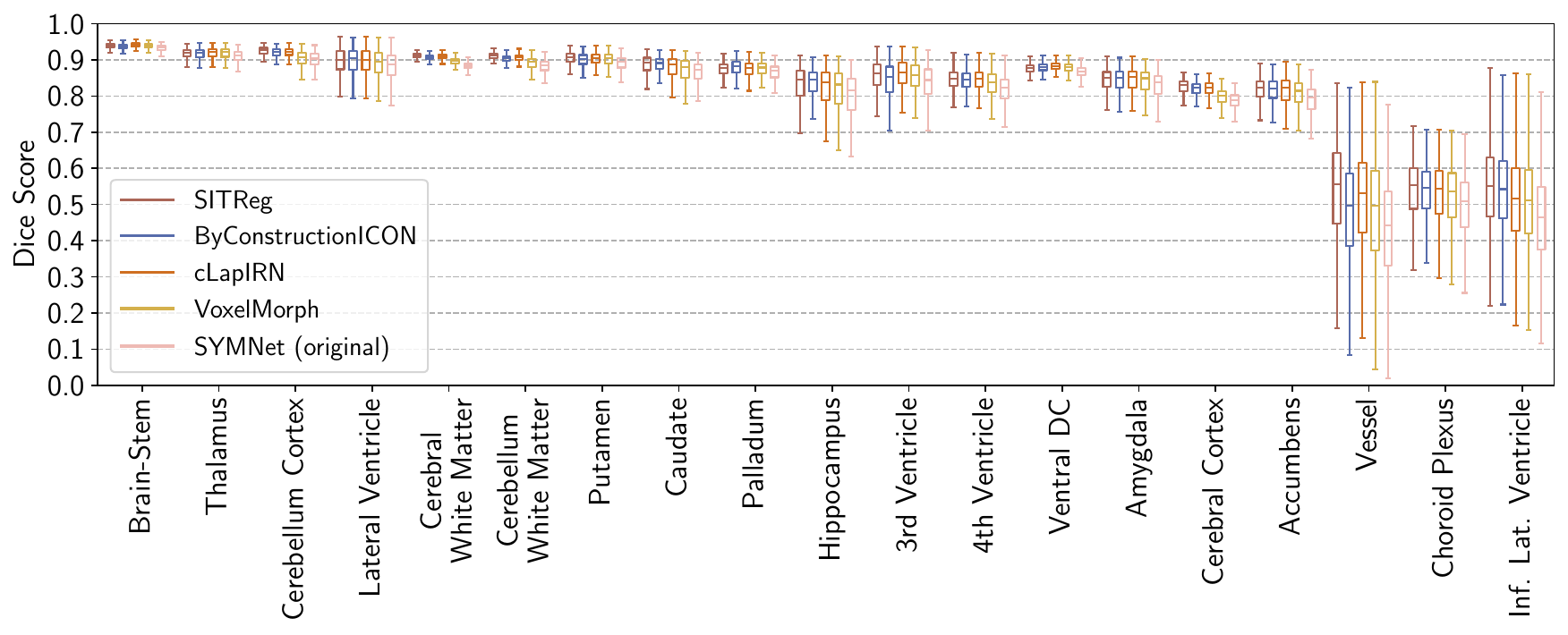}
\caption{\textbf{Individual brain structure dice scores for the OASIS experiment.} Boxplot shows performance of each of the compared methods on each of the brain structures in the OASIS dataset. Algorithms from left to right in each group: SITReg, ByConstructionICON, cLapIRN, VoxelMorph, SYMNet (original)}
\label{appendix-fig:dice_scores_oasis}
\end{figure}

\begin{figure}[h]
\centering
\includegraphics[width=1.0\textwidth]{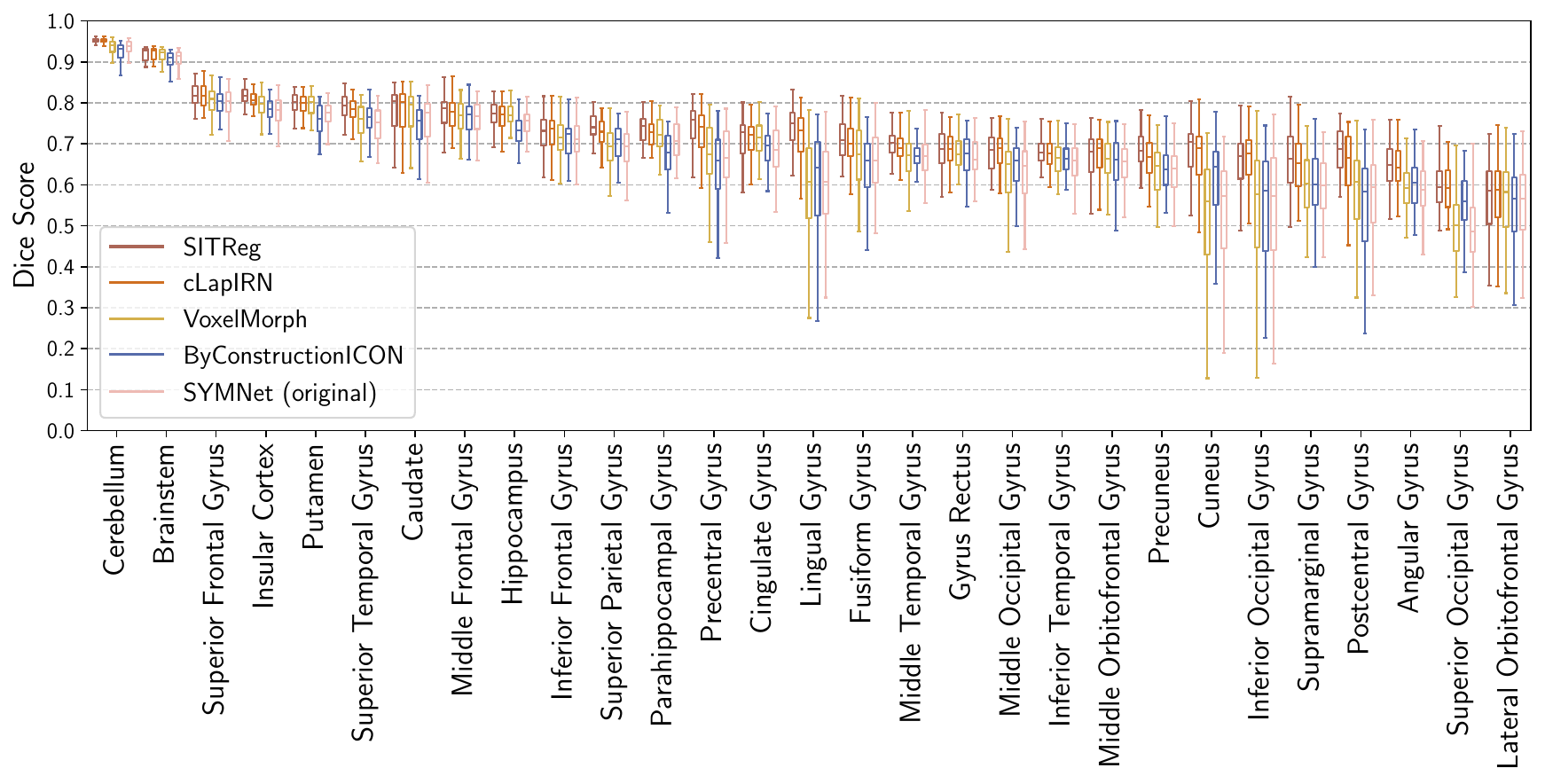}
\caption{\textbf{Individual brain structure dice scores for the LPBA40 experiment.} Boxplot shows performance of each of the compared methods on each of the brain structures in the LPBA40 dataset. Algorithms from left to right in each group: SITReg, cLapIRN, VoxelMorph, ByConstructionICON, SYMNet (original)}
\label{appendix-fig:dice_scores_lpba40}
\end{figure}

\begin{figure}[h]
\centering
\includegraphics[width=0.96\textwidth]{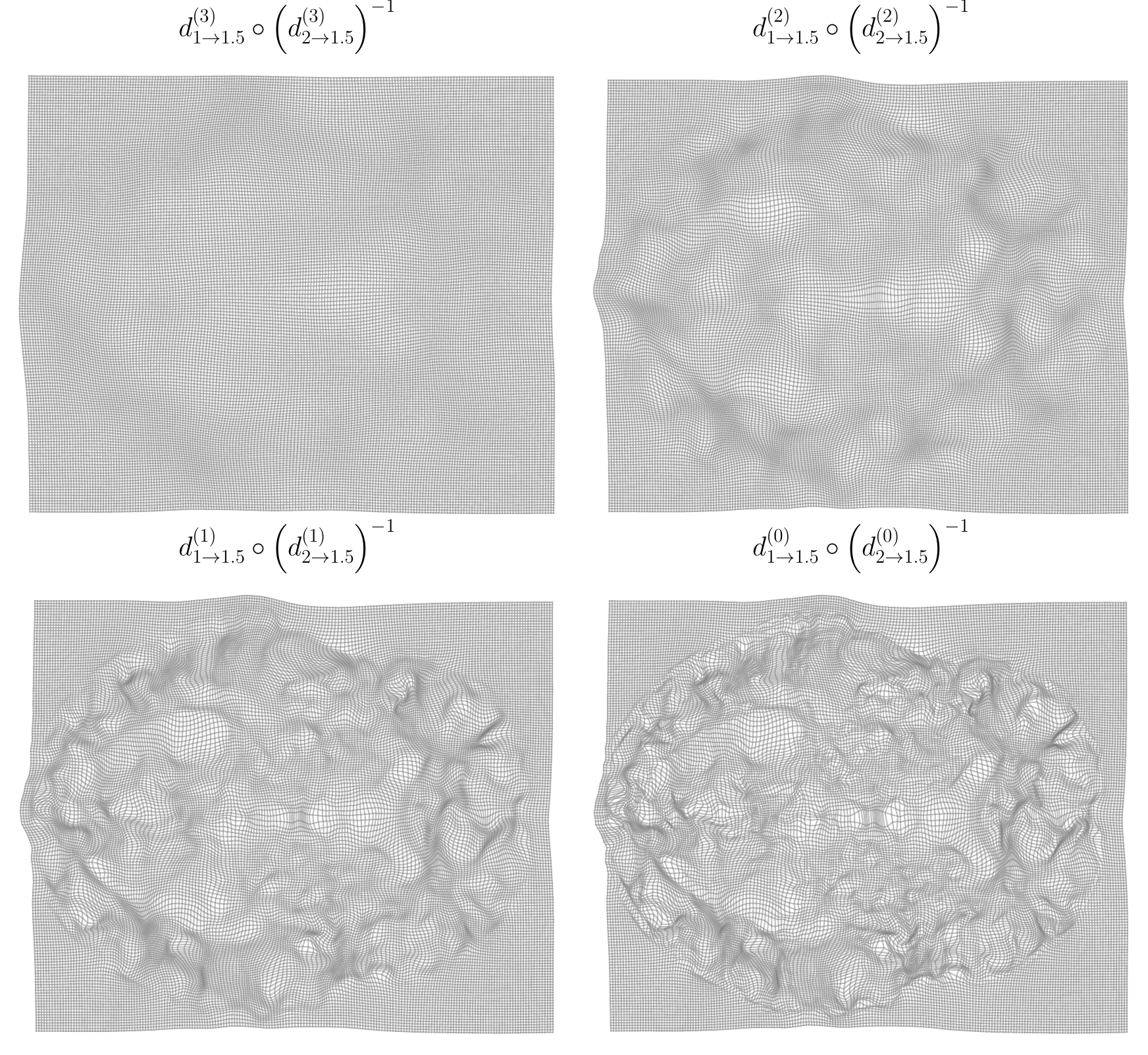}
\caption{\textbf{Visualization of deformation being gradually updated.} Each $d_{1\to1.5}^{(k)} \circ \left( d_{2\to1.5}^{(k)} \right)^{-1}$ corresponds to the full deformation after resolution level $k$. The example is from the OASIS experiment.}
\label{appendix-fig:detailed_deformation_example}
\end{figure}
\clearpage

\section{Details on statistical significance}\label{appendix:statistical_significance}

We computed statistical significance of the results comparing the test set predictions of the trained models with each other. We measured the statistical significance using permutation test, and in practice sampled $10000$ permutations.

To establish for certain the relative performance of the methods with respect to the tight metrics, one should train multiple models per method with different random seeds. However, our claim is not that the developed method improves the results with respect to a single tight metric but rather that the overall performance is better by a clear margin (see Section \ref{sec:discussion}).

\section{Clarifications on symmetry, inverse consistency, and topology preservation}\label{appendix:property_examples}

Here we provide examples of symmetry, inverse consistency and lack of topology preservation to further clarify how the terms are used in the paper.

Since symmetry and inverse consistency are quite similar properties, their exact difference might remain unclear. Examples of registration methods that are \textit{inverse consistent by construction but not symmetric} are many deep learning frameworks applying the stationary velocity field \citep{arsigny2006log} approach, e.g,  \citep{dalca2018unsupervised, krebs2018unsupervised, krebs2019learning, mok2020fast}. All of them use a neural network to predict a velocity field for an ordered pair of input images. The final deformation is then produced via Lie algebra exponentiation of the velocity field, that is, by integrating the velocity field over itself over unit time. Details of the exponentiation are not important here but the operation has an interesting property: By negating the velocity field to be exponentiated, the exponentiation results in inverse deformation. Denoting the Lie algebra exponential by $\exp$, and using notation from Section \ref{sec:intro}, we can define such methods as
\begin{equation}
    \begin{cases}
        f_{1\to2}(x_A, x_B) &:= \exp(g(x_A, x_B))\\
        f_{2\to1}(x_A, x_B) &:= \exp(-g(x_A, x_B))
    \end{cases}
\end{equation}
where $g$ is the learned neural network predicting the velocity field. As a result, the methods are inverse consistent by construction since $\exp(g(x_A, x_B)) = \exp(-g(x_A, x_B))^{-1}$ (accuracy is limited by spatial sampling resolution). However, by changing the order of inputs to $g$, there is no guarantee that $g(x_A, x_B) = -g(x_B, x_A)$ and hence such methods are not symmetric by construction.

MICS \citep{estienne2021mics} is an example of a method which is \textit{symmetric by construction but not inverse consistent}. MICS is composed of two components: encoder, say $E$, and decoder, say $D$, both of which are learned. The method can be defined as
\begin{equation}
    \begin{cases}
        f_{1\to2}(x_A, x_B) &:= D(E(x_A, x_B) - E(x_B, x_A))\\
        f_{2\to1}(x_A, x_B) &:= D(E(x_B, x_A) - E(x_A, x_B)).
    \end{cases}
\end{equation}
As a result, the method is symmetric by construction since $f_{1\to2}(x_A, x_B) = f_{2\to1}(x_B, x_A)$ holds exactly. However, there is no architectural guarantee that $f_{1\to2}(x_A, x_B)$ and $f_{2\to1}(x_B, x_A)$ are inverses of each other, and the paper proposes to encourage that using a loss function. In the paper they use such components in multi-steps manner, and as a result the overall architecture is no longer symmetric.

Lack of topology preservation means in practice that the predicted deformation folds on top of itself. An example of such deformation is shown in Figure \ref{appendix-fig:folding_deformation_example}.

\begin{figure}[h]
\centering
\includegraphics[width=0.6\textwidth]{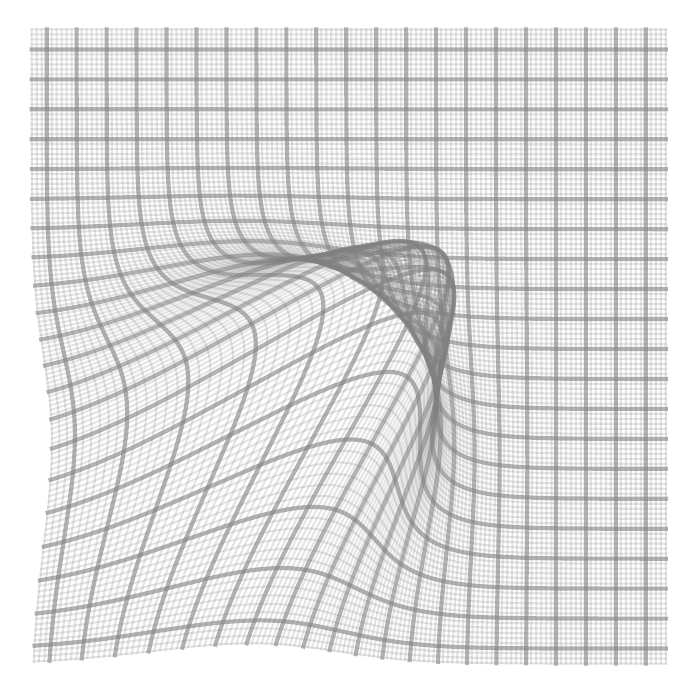}
\caption{\textbf{Visualization of a 2D deformation which is not topology preserving.} The deformation can be seen folding on top of itself.}
\label{appendix-fig:folding_deformation_example}
\end{figure}

\end{document}